\documentclass[11pt]{article}
\usepackage{amsmath, amsthm, amssymb, bbm, xparse, xspace}
\usepackage[margin=1in]{geometry}
\usepackage[round]{natbib}
\usepackage[dvipsnames,table]{xcolor}
\usepackage{hyperref}
\hypersetup{
    colorlinks=true,
    linkcolor=blue,
    citecolor=blue,
    urlcolor=blue,
    filecolor=blue
}
\usepackage[capitalise]{cleveref}
\crefname{subsection}{subsection}{subsections}
\crefname{assumption}{assumption}{assumptions}
\crefname{example}{example}{examples}
\usepackage{setspace}
\usepackage[shortlabels]{enumitem}
\usepackage{graphicx}
\usepackage{booktabs}
\usepackage{float}
\usepackage{multirow}
\usepackage{colortbl}
\usepackage{etoolbox}
\usepackage[most]{tcolorbox}
\usepackage{algorithm}
\usepackage{algpseudocode}
\usepackage{tikz}
\usepackage{subcaption}
\usetikzlibrary{arrows.meta, positioning, shapes.geometric, fit, calc}

\usepackage[english]{babel}
\usepackage[autostyle, english = american]{csquotes}
\MakeOuterQuote{"}

\theoremstyle{plain}
\newtheorem{theorem}{Theorem}

\newtheorem{corollary}{Corollary}

\theoremstyle{definition}

\newtheorem{example}{Example}

\newcommand{\bE}{\mathbb{E}}
\newcommand{\bP}{\mathbb{P}}
\newcommand{\bR}{\mathbb{R}}

\usepackage{color-edits}
\addauthor{Will}{blue}
\addauthor[Jackie]{jb}{magenta}
\addauthor[Tianyi]{tp}{green}
\addauthor{Note}{red}

\title{AI Agents for Inventory Control: \\ Human-LLM-OR Complementarity}

\author{Jackie Baek\thanks{Stern School of Business, New York University, \href{mailto:baek@stern.nyu.edu}{baek@stern.nyu.edu}}
\and
Yaopeng Fu\thanks{Columbia University, \href{mailto:yf2726@columbia.edu}{yf2726@columbia.edu}}
\and
Will Ma\thanks{Graduate School of Business and Data Science Institute, Columbia University, \href{mailto:wm2428@gsb.columbia.edu}{wm2428@gsb.columbia.edu}}
\and
Tianyi Peng\thanks{Graduate School of Business and Data Science Institute, Columbia University, \href{mailto:tianyi.peng@columbia.edu}{tianyi.peng@columbia.edu}}
}
\date{}

\begin{document}
\maketitle

\begin{abstract}
Inventory control is a fundamental operations problem in which ordering decisions are traditionally guided by theoretically grounded operations research (OR) algorithms.  However, such algorithms often rely on rigid modeling assumptions and can perform poorly when demand distributions shift or relevant contextual information is unavailable. Recent advances in large language models (LLMs) have generated interest in AI agents that can reason flexibly and incorporate rich contextual signals, but it remains unclear how best to incorporate LLM-based methods into traditional decision-making pipelines.

We study how OR algorithms, LLMs, and humans can interact and complement each other in a multi-period inventory control setting. We construct \textsc{InventoryBench}, a benchmark of over 1{,}000 inventory instances spanning both synthetic and real-world demand data, designed to stress-test decision rules under demand shifts, seasonality, and uncertain lead times. Through this benchmark, we find that OR-augmented LLM methods outperform either method in isolation, suggesting that these methods are complementary rather than substitutes.

We further investigate the role of humans through a controlled classroom experiment that embeds LLM recommendations into a human-in-the-loop decision pipeline. Contrary to prior findings that human-AI collaboration can degrade performance, we show that, on average, human-AI teams achieve higher profits than either humans or AI agents operating alone.  Beyond this population-level finding, we formalize an \emph{individual-level complementarity effect} and derive a distribution-free lower bound on the fraction of individuals who benefit from AI collaboration; empirically, we find this fraction to be substantial.

Taken together, our results provide evidence that effective inventory management benefits from a complementary system in which OR provides heuristics and calculators, LLMs contribute world knowledge and contextual reasoning, and human judgment adds value beyond automated decision making.
\end{abstract}

\section{Introduction} \label{sec:introduction}
Inventory control is the fundamental problem of ordering supply to match demand.
To solve this problem, the field of operations research (OR) has developed a rich toolkit of heuristics and approximations for handling supply and demand uncertainty, forming the backbone of modern inventory management systems (see \Cref{sec:inv_mgmt}).
These methods---base-stock policies, newsvendor models, $(s, S)$ rules, and their data-driven variants---are theoretically grounded, computationally efficient, and widely deployed in practice.
However, they rely on fixed modeling assumptions (e.g., stationary demand, known lead times) and on historical data that may not reflect the current environment.
When demand distributions shift, contextual information is difficult to encode in a formal model, or supply conditions change unexpectedly, these heuristics can perform poorly.
For this reason, the core of modern inventory management is not purely algorithmic but human-in-the-loop \citep[see e.g.][]{boute2021digital}: OR tools generate recommendations, and human experts make the final call, applying contextual judgment to accept, adjust, or override what the algorithm suggests.

Recent advances in large language models (LLMs) introduce a fundamentally different form of reasoning into this pipeline. Unlike traditional OR algorithms, LLMs can incorporate natural-language context---product descriptions, calendar information, qualitative market signals---and draw on broad world knowledge to form reasonable judgments even when the formal model is incomplete or the environment shifts in unexpected ways. At the same time, LLMs are a nascent technology for operational decision-making, and it remains unclear when they help, how they should be deployed alongside existing OR heuristics, and how they interact with human decision makers.

This paper studies the \emph{complementarity} among OR algorithms, LLM-based agents, and humans in a multi-period inventory control setting.
The ways in which these three components can be combined are diverse: an OR heuristic can feed a structured recommendation to an LLM, which decides whether to follow or override it; alternatively, an LLM can generate demand forecasts or contextual estimates that an OR algorithm acts on.
When humans enter the loop, further design choices arise: a human can retain full decision authority and use AI-generated recommendations as advisory input, or delegate operational decisions to an AI system while providing high-level strategic guidance at periodic checkpoints.
Understanding which methods work best, and why, requires systematic experimentation across a range of problem environments.
We organize our investigation into two parts.

\textbf{Part 1: OR $\times$ LLM interactions.}
We begin by understanding the interaction between OR and LLM without a human in the loop.
We consider four methods of combining these tools:
\begin{itemize}[nosep]
    \item \textbf{OR} alone: a simple OR pipeline that does data-driven demand estimation followed by a well-established inventory heuristic;
    \item \textbf{LLM} alone: the LLM directly makes all inventory decisions;
    \item \textbf{OR$\to$LLM}: the OR pipeline above is described to an LLM, who treats its output as a recommendation but can override it in the final decision;
    \item \textbf{LLM$\to$OR}: an LLM estimates all uncertain parameters such as demand, and these are fed into the well-established inventory heuristic to make the final decision.
\end{itemize}
To understand the strengths and weaknesses of each method, we construct \textsc{InventoryBench}, a benchmark of 1{,}320 instances (720~based on synthetic data, 600~based on real data) spanning non-stationary demand patterns (changepoints, trends, seasonality), varying cost structures (leading to different under- vs.\ over- stocking tradeoffs), and three lead-time regimes (zero, fixed, and stochastic with lost orders).
We evaluate all four methods using three frontier LLMs (Gemini~3~Flash, Grok~4.1~Fast, GPT-5~Mini). 
Table~\ref{tab:intro_1a} summarizes the performance for Gemini~3~Flash.

\begin{table}
\caption{Summary of normalized reward (mean $\pm$ 95\% CI). Higher is better.}
\label{tab:intro_summary}
\centering
\small
\begin{minipage}[t]{0.48\textwidth}
\subcaption{OR $\times$ LLM: 1{,}320 \textsc{InventoryBench} instances (Gemini 3 Flash).}\label{tab:intro_1a}
\centering
\begin{tabular}{lcc}
\toprule
Method & Mean & 95\% CI \\
\midrule
OR & 0.445 & $\pm$0.018 \\
LLM & 0.494 & $\pm$0.018 \\
OR $\to$ LLM & \textbf{0.538} & $\pm$0.016 \\
LLM $\to$ OR & 0.501 & $\pm$0.018 \\
\bottomrule
\end{tabular}
\end{minipage}%
\hfill
\begin{minipage}[t]{0.48\textwidth}
\subcaption{Human-in-the-loop: classroom experiment (69 participants, 3 instances).}\label{tab:intro_1b}
\centering
\begin{tabular}{lcc}
\toprule
Setting & Mean & 95\% CI \\
\midrule
Mode A (OR$\to$Human) & 0.466 & $\pm$0.025 \\
Mode B (OR$\to$LLM$\to$Human) & \textbf{0.534} & $\pm$0.020 \\
Mode C (OR$\to$LLM + Guidance) & 0.464 & $\pm$0.018 \\
\midrule
OR (no human) & 0.376 & --- \\
OR$\to$LLM (no human) & 0.482 & $\pm$0.007 \\
\bottomrule
\end{tabular}
\end{minipage}
\vspace{-1.6em}
\end{table}

We find that combining OR and LLM is broadly beneficial: both combination methods (OR$\to$LLM and LLM$\to$OR) outperform either method in isolation.
The OR$\to$LLM pipeline achieves the best overall performance (0.538), a 21\% improvement over OR alone.
Drilling into the results, the two components have complementary strengths.
LLMs excel at detecting demand shifts, incorporating product-specific world knowledge (e.g., seasonality of swimwear), and identifying supply disruptions such as lost orders---capabilities absent from the OR heuristic.\footnote{One could, in principle, augment the OR method's data-driven estimation with statistical changepoint detection, trend fitting, or other extensions, but it is difficult to capture all possibilities in one general methodology.  Our goal here is to let LLMs represent that "one general methodology", and contrast it with the precision of the simple data-driven estimation.}
OR, in turn, provides the mathematical precision needed for base-stock calculations under stable conditions and long lead times, whereas LLMs can struggle with inventory pipeline tracking and sometimes falsely identify demand shifts when it is pure noise.
Our results also suggest that LLMs alone are less "calibrated" to the under- vs.\ over- stocking tradeoffs implied by different cost structures, echoing the message of \citet{liu2025large}.

\textbf{Part 2: Human-in-the-loop experiment.}
Next, we study how humans should interact with these tools.
We build a web-based inventory game and conduct a controlled classroom experiment with 69 participants.
Each participant plays three real-data instances under three collaboration modes with varying decision authority:
\begin{itemize}[nosep]
\item \textbf{Mode~A} (OR$\to$Human): the human treats the output of the OR pipeline as a recommendation, and makes the final decision;
\item \textbf{Mode~B} (OR$\to$LLM$\to$Human): the human treats the output of the OR$\to$LLM method as a recommendation, reads the LLM's reasoning, and then makes the final decision;
\item \textbf{Mode~C} (OR$\to$LLM + Human Guidance): the LLM decides as in OR$\to$LLM, but the human can provide strategic guidance every few periods.
\end{itemize}
Figure~\ref{fig:modes} shows the decision panels for each mode.
Mode--instance assignments are randomized so that each participant experiences each mode once.

\begin{figure}
    \centering
    \begin{subfigure}[b]{0.32\textwidth}
        \includegraphics[width=\textwidth]{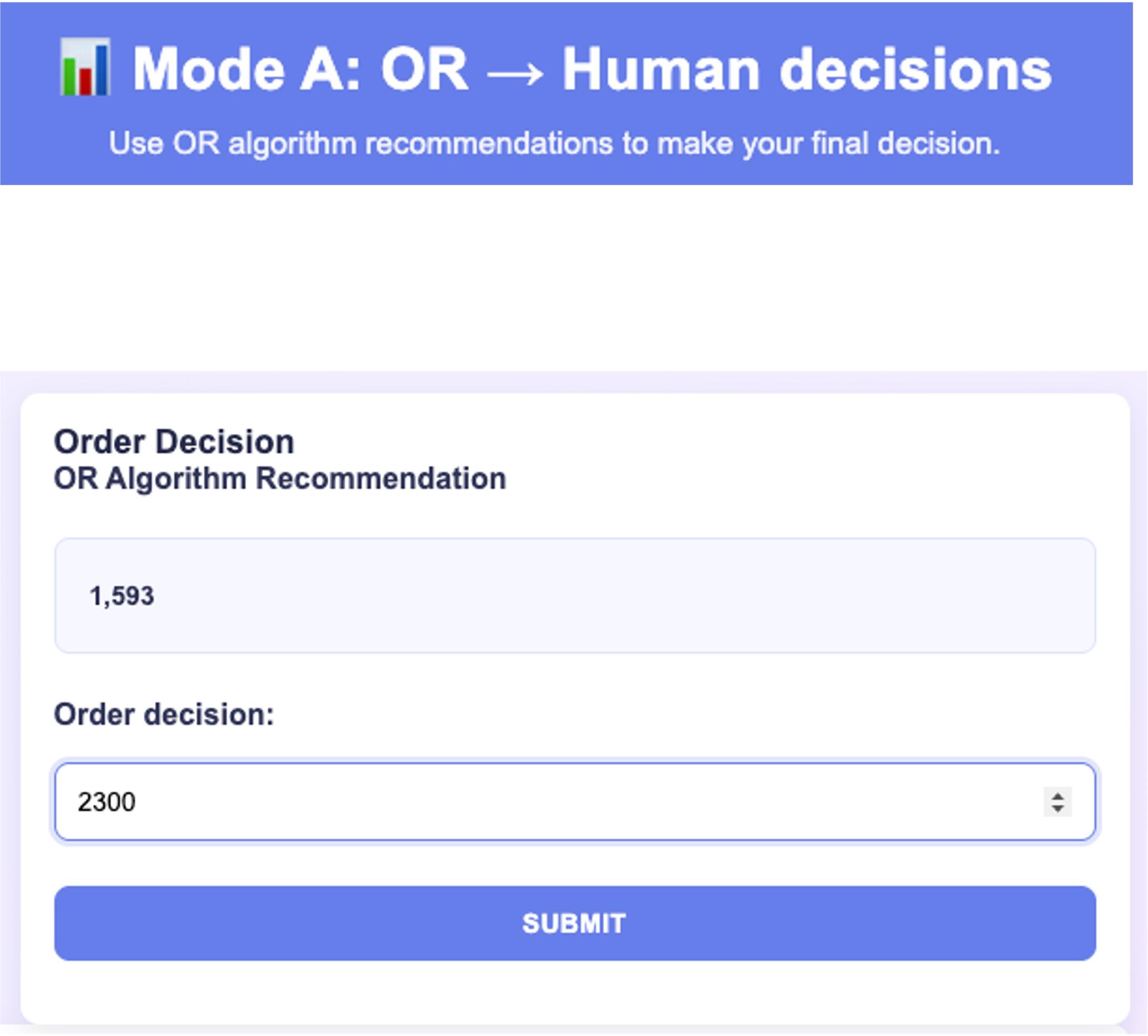}
        \caption{Mode~A}
        \label{fig:modeA}
    \end{subfigure}
    \hfill
    \begin{subfigure}[b]{0.32\textwidth}
        \includegraphics[width=\textwidth]{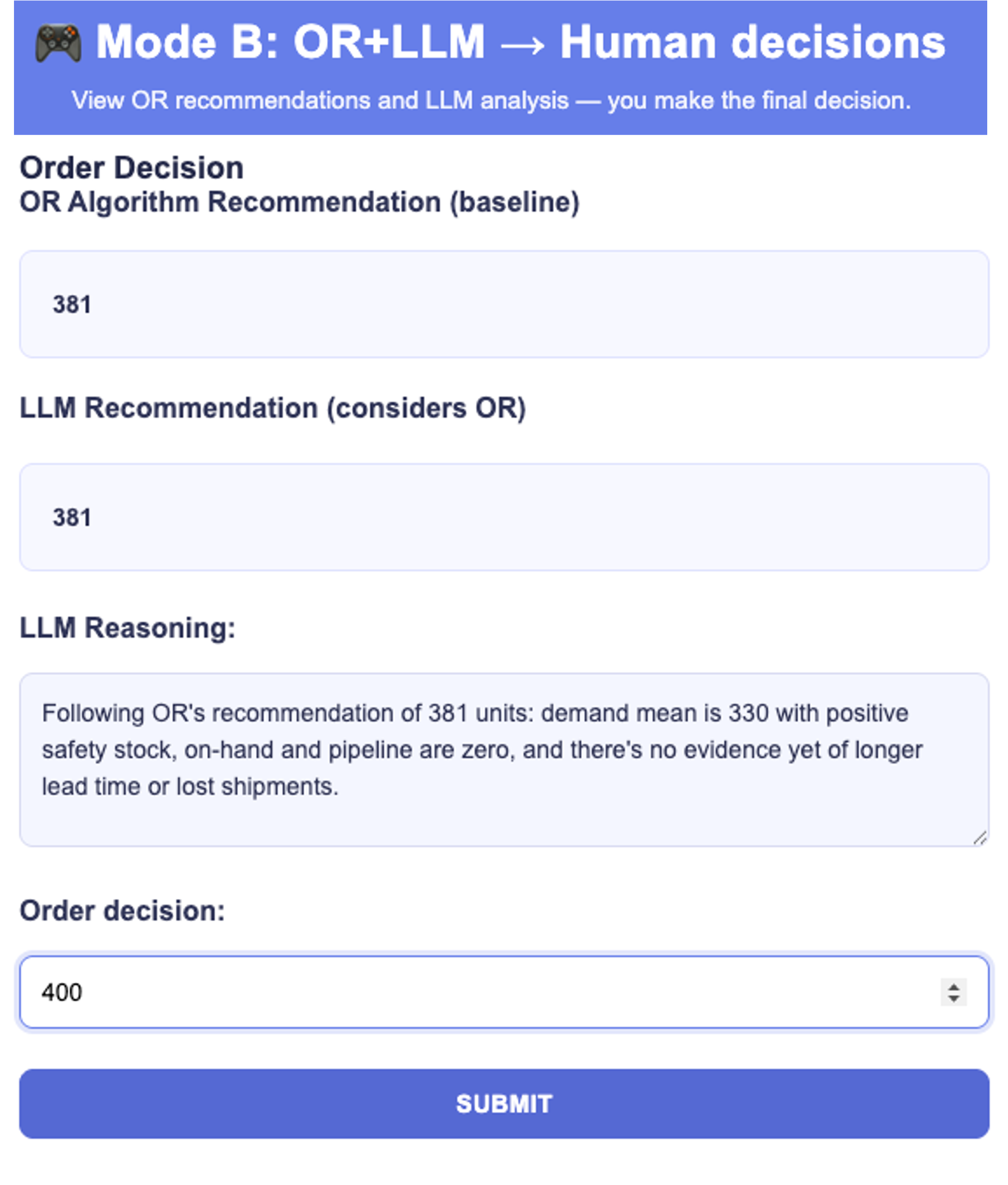}
        \caption{Mode~B}
        \label{fig:modeB}
    \end{subfigure}
    \hfill
    \begin{subfigure}[b]{0.32\textwidth}
        \includegraphics[width=\textwidth]{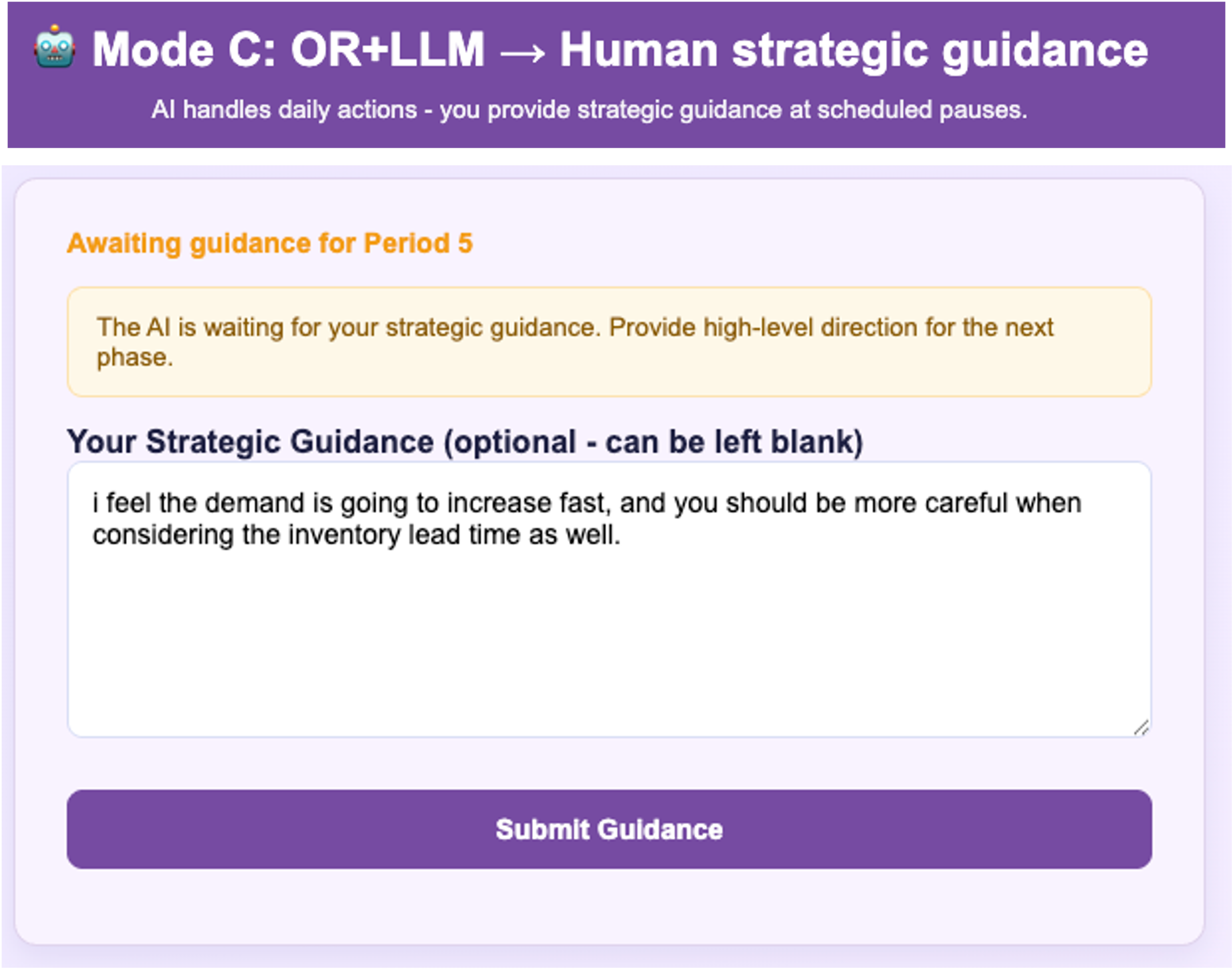}
        \caption{Mode~C}
        \label{fig:modeC}
    \end{subfigure}
    \caption{Decision panels for the three collaboration modes.
    \textbf{Mode~A}: the participant sees the OR recommendation and enters a final order.
    \textbf{Mode~B}: the participant sees the OR-augmented LLM recommendation and reasoning before entering a final order.
    \textbf{Mode~C}: the AI makes ordering decisions autonomously; the participant provides optional strategic guidance at scheduled pauses (every 4 periods).}
    \label{fig:modes}
\end{figure}

Table~\ref{tab:intro_1b} summarizes the results.
Mode~B achieves the best overall performance, significantly outperforming both Mode~A and Mode~C.
Crucially, human collaboration modes also outperform their automated counterparts: Mode~A outperforms OR alone and Mode~B outperforms OR$\to$LLM, demonstrating that human judgment adds value beyond what automated methods alone can achieve.
(These differences are all statistically significant according to a pre-registered experiment, as we describe in \Cref{sec:game_results}.)

Beyond this population-level finding, we ask whether \emph{individual} participants genuinely benefit from AI collaboration, or whether the aggregate gains simply reflect a selection effect where weaker participants follow the AI and stronger ones ignore it, without anyone actually exceeding what they or the AI could achieve alone.

We formalize a notion of \emph{individual-level complementarity}: whether a person's collaborative performance exceeds the better of their solo performance and the AI's solo performance. 
Because each participant is observed in only one condition (with or without AI), we cannot directly measure this for any specific individual. 
However, we prove a theorem (\cref{thm:ce_lower_bound}) that provides a distribution-free lower bound on the fraction of individuals who experience positive complementarity. 
The key insight is that this bound can be estimated by comparing the distribution of collaborative performance to the distribution of solo performance across different people, without needing to observe the same person under both conditions.
Applying this bound to our data separately for each instance, we estimate that at least 30\% to 60\% of individuals experience strictly positive complementarity, depending on the instance.

\textbf{Overall.}
Taken together, our results provide a systematic study of the emerging question of how OR algorithms, LLMs, and humans should interact for operational decision-making (Figure~\ref{fig:complementarity}).
On one hand, LLMs add substantial value to inventory control: they detect demand regime changes, incorporate world knowledge, and identify supply disruptions---capabilities that are generally difficult for traditional OR algorithms.
On the other hand, the traditional players remain essential: OR algorithms provide the mathematical precision that LLMs lack for base-stock calculations, and humans provide situational judgment and serve as a safeguard---our analysis shows that humans are able to step in precisely when the LLM fails to reason about the situation correctly, e.g.\ when they fail to identify lost orders.
The finding that humans add value is particularly noteworthy: achieving genuine human–AI complementarity is widely recognized as difficult, with systematic evidence showing that human–AI teams frequently fail to outperform the better of human and AI acting independently \citep{bansal2021does,vaccaro2024combinations,hemmer2025complementarity}.
Our results suggest that inventory control---with its sequential decisions, delayed feedback, and rich contextual signals---is a setting where human judgment genuinely complements AI.

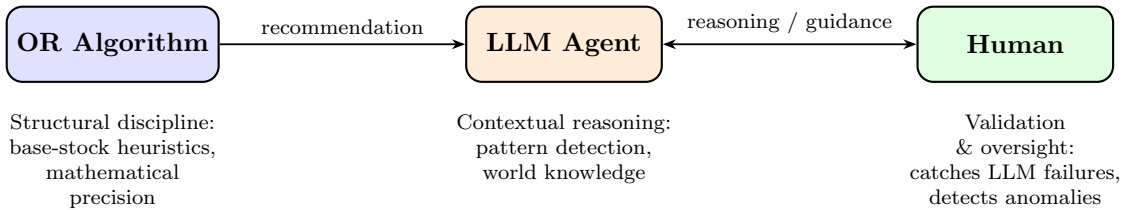
\begin{figure}
\centering
\begin{tikzpicture}[
    agent/.style={draw, rounded corners=6pt, minimum width=2.6cm, minimum height=1.0cm, align=center, font=\small\bfseries, line width=0.8pt},
    capability/.style={font=\scriptsize, align=center, text width=3cm},
    arrow/.style={-{Stealth[length=5pt]}, line width=0.7pt},
    doublearrow/.style={{Stealth[length=5pt]}-{Stealth[length=5pt]}, line width=0.7pt},
]

\node[agent, fill=blue!12] (OR) at (0, 0) {OR Algorithm};
\node[agent, fill=orange!15] (LLM) at (6, 0) {LLM Agent};
\node[agent, fill=green!12] (Human) at (12, 0) {Human};

\node[capability, below=8pt of OR] {Structural discipline:\\base-stock heuristics,\\mathematical precision};
\node[capability, below=8pt of LLM] {Contextual reasoning:\\pattern detection,\\world knowledge};
\node[capability, below=8pt of Human] {Validation \& oversight:\\catches LLM failures,\\detects anomalies};

\draw[arrow] (OR) -- node[above, font=\scriptsize] {recommendation} (LLM);
\draw[doublearrow] (LLM) -- node[above, font=\scriptsize] {reasoning / guidance} (Human);

\end{tikzpicture}
\caption{The complementary interaction pattern of OR algorithms, LLM agents, and humans. %
}
\label{fig:complementarity}
\end{figure}
\textbf{Open-source benchmark and game.}
As a by-product, we release two resources.
First, \textsc{InventoryBench}\footnote{\textsc{InventoryBench} is available at \url{https://tianyipeng.github.io/InventoryBench/}.}---our benchmark of 1{,}320 inventory instances, together with a public leaderboard to track progress---is available for the community to test and compare frontier LLMs, fine-tuning approaches, advanced OR heuristics, and other machine learning methods---we welcome researchers to try and share their approaches.
Second, we open-source the web-based \textsc{AI-Human Inventory Game}\footnote{\textsc{AI-Human Inventory Game} is available at \url{https://github.com/TianyiPeng/AI-human-inventory-game.git}.} used in our experiment (Figure~\ref{fig:game_interface}), which can serve as a basis to be further developed for teaching and research in supply chain management.

\paragraph{Roadmap and summary of contributions.}
We design LLM agents for inventory control (\Cref{sec:algorithms}), construct a benchmark of 1{,}320 instances to evaluate them (\Cref{sec:instances}), and build a web-based game for human--AI interaction (\Cref{sec:classroom_setup}).
\emph{Findings without humans.}
The OR$\to$LLM pipeline achieves the best overall performance (\Cref{sec:metric_overall_result}), but disaggregating by lead time (\Cref{sec:disagg_lead_time}), instance family (\Cref{sec:disagg_instance_family}), and critical fractile (\Cref{sec:calibration_cr}) reveals the complementary strengths of each component: OR excels at handling long deterministic lead times and avoids overfitting to noise, while LLMs excel under stochastic lead times and at detecting sudden demand shifts and leveraging world knowledge.
\emph{Findings with humans.}
Pre-registered classroom experiments yield clean, statistically significant separations between collaboration modes (\Cref{sec:game_results}), demonstrating that both the LLM and the human add value to the decision pipeline.
We develop a theoretical framework for measuring human--AI complementarity (\Cref{sec:complementarity_theory}) and apply it to establish individual-level complementarity in our setting (\Cref{sec:complementarity}).
Finally, we provide illustrative anecdotes of LLM reasoning (\Cref{sec:llm_reasoning_anecdotes}) and analyze the mechanisms through which humans improve upon LLM recommendations (\Cref{sec:mechanisms}).

\section{Related Work} \label{sec:relatedWork}

\subsection{Supply Chain and Inventory Management} \label{sec:inv_mgmt}

Inventory can be the first topic mentioned in operations and supply chain management textbooks \citep[e.g.]{SimchiLevi2008Designing,snyder2019fundamentals}, and is also a problem in which the uncertainty in supply and demand can be used to test the judgment and contextual knowledge of humans or LLMs.  We consider a canonical inventory model in which there are lead times and lost sales (i.e., unmet demands cannot be recouped), both the norm in practice, and this exact model is widely deployed at modern enterprises to recommend inventory decisions \citep{qi2023practical,madeka2022deep,liu2023ai,xie2025deepstock}, noting that it is generally unnecessary to optimize different items' inventories jointly.
Academically, even optimizing for a single item is notoriously difficult \citep{Zipkin2008OldNew,Zipkin2008Structure}, which is why there is a large OR literature on heuristics for this problem \citep{BijvankVis2011Review,HuhJanakiraman2009OrderUpTo,GoldbergEtAl2016Asymptotic}.
We make our heuristic of choice the capped base-stock policy \citep{Xin2021CappedBaseStock}, a recent development that is simple and powerful, and considered the state of the art in several papers since \citep{XinGoldberg2022Approx,LyuZhangXin2024UCB,alvo2023neural}.

We note that some of the aforementioned papers use Reinforcement Learning to solve this inventory model, often with high-dimensional covariates.  We are essentially using the LLM's contextual knowledge in place of these high-dimensional covariates, leaving the integration of LLMs with Reinforcement learning to future work.

\subsection{LLMs for Operational Decisions}

There is a growing literature on using LLMs to
formulate optimization problems \citep{ahmaditeshnizi2024optimus,zhou2025auto,huang2025orlm} 
or to directly
make operational decisions \citep{LongSimchiLeviCalmonCalmon2025Autonomous,backlund2025vendingbench,kumar2025performance,fish2025econevals}.
Our paper lies somewhere in-between, where we are relying on the structure and calculation of simple OR heuristics, and testing the ability of LLMs to integrate contextual world knowledge into them.  More precisely, our paper differs from the literature by: (i) considering OR heuristics instead of exact optimization solvers; (ii) exploiting the world knowledge and generality of LLMs to adapt to changing or anomalous conditions, instead of using them to do the formulation of the optimization problem.
Our punchline that OR+LLM is best echoes what is found in \citet{duan2025ask,baek2026evaluatingllmpersonagenerateddistributions}, but in very different contexts---\citet{duan2025ask} also study inventory but use the LLM to extract uncertain parameters (e.g., holding cost) instead of make daily decisions, while \citet{baek2026evaluatingllmpersonagenerateddistributions} focus on single-period problems and use the LLM only to generate data for a standard downstream optimization.
More broadly, the power of combining OR with LLMs has been discussed in \citet{dai2025assured,cohen2025supply,hu2025ai}.

\subsection{Human-AI Decision-making for Operations}

Importantly, our paper tests how OR+LLM integrates with humans, and to the best of our knowledge, breaks new ground by testing this in a controlled lab experiment.  In particular, we compare Mode A (OR$\to$Human) to Mode B (OR$\to$LLM$\to$Human), which empirically measures the value add of the LLM when a human makes the final decision, even though this idea has been seen in enterprise settings long before \citep{li2023large}.
The indispensability of humans in operational pipelines is discussed in \citet{boute2021digital}, and they distinguish between human control vs.\ oversight, which is exactly what we test in Modes B vs.\ C, in an inventory management problem.

A broader literature shows that achieving genuine human–AI complementarity is difficult. Surveys of empirical studies find mixed results \citep{lai2023towards}, a meta-analysis finds that human–AI teams frequently fail to outperform the better of the two components operating independently \citep{vaccaro2024combinations}.
Recent works have formalized theoretical models of complementarity to design better systems or to provide fundamental limits of human-AI performance 
\citep{wilder2020learning,donahue2022human,peng2025no,guo2024decision,guo2025value,hemmer2025complementarity}.
We add to this literature by formalizing notions of population and individual-level complementarity effects.  

Within operations, a line of work studies how humans interact with algorithms in practice.
\citet{fildes2009effective} analyze demand forecasts from four supply-chain companies and find that small judgmental adjustments to algorithmic forecasts tend to decrease accuracy, while larger adjustments tend to improve it.
\citet{sun2022predicting} study packing decisions at a warehouse and demonstrate the value of an algorithm designed to anticipate and incorporate human deviations.
A complementary stream studies how to improve better human--AI systems through modifying how human decisions are elicited \citep{ibrahim2021eliciting}, providing algorithm transparency \citep{balakrishnan2026human}, modifying systems loads \citep{snyder2026algorithm} or modifying the algorithms \citep{mclaughlin2024designing,bastani2026improving,grand2026best}.

\section{Problem Formulation, Algorithms, and Instances} \label{sec:game_description}

We study a standard (see \Cref{sec:inv_mgmt}) multi-period inventory control problem in which a decision maker manages a single product over a finite horizon of $T$ periods.

\paragraph{State and dynamics.}
In each period $t = 1, 2, \dots, T$, events occur in the following sequence:
\begin{enumerate}
    \item \textbf{Observe state and context.} The decision maker knows the on-hand inventory $I_t$, the history of past demands, the history of past orders and the arrival times of the orders that have arrived, and contextual information $x_t$ (a free-form text string that may include static product descriptions and calendar dates, which give information about upcoming demand).
    \item \textbf{Place an order.} The decision maker chooses an order quantity $q_t \ge 0$.
    \item \textbf{Receive arrivals.} Shipments from earlier orders arrive. Each order placed at time $\tau \le t$ has a lead time $\ell_\tau \in \{0,1,2,\ldots\} \cup \{\infty\}$, which is deterministic but unknown to the decision maker at the time the order is placed. If $\tau + \ell_\tau = t$, then order $q_\tau$ arrives and is added to on-hand inventory; if $\ell_\tau = \infty$, the order never arrives (is lost).
    Let $A_t = \sum_{\tau \le t : \tau + \ell_\tau = t} q_\tau$ denote total arrivals in period $t$.
    \item \textbf{Realize demand.} Demand $d_t$ is realized and observed (even under a stockout, i.e., we assume \emph{uncensored} demands). Sales equal $s_t = \min\{d_t,\, I_t + A_t\}$, and unsatisfied demand is lost.
    \item \textbf{Update inventory.} Leftover inventory carries over to the next period:
    \[
    I_{t+1} \;=\; I_t + A_t - s_t \;=\; \max\{0,\, I_t + A_t - d_t\}.
    \]
\end{enumerate}

\paragraph{Objective.} The per-period profit equals sales revenue $p \cdot s_t$ minus holding cost $h \cdot I_{t+1}$, where $p > 0$ is the fixed per-unit profit margin and $h > 0$ is the fixed per-unit per-period holding cost.
The decision maker's objective is to maximize total profit over the horizon,
\[
\sum_{t=1}^T \left( p \cdot s_t \;-\; h \cdot I_{t+1} \right).
\]
An important parameter to highlight is the \emph{critical fractile} $\rho:=p/(p+h)$, which is a ratio in $[0,1]$ indicating how desirable it is to overstock (instead of understock).

\paragraph{Initialization.}
Each instance begins with five realized demand values $\{d_{-4}, d_{-3}, d_{-2}, d_{-1}, d_0\}$ provided as historical data.
The initial state is $I_1 = 0$ (no on-hand inventory) with no in-transit orders.
The decision maker is informed of the profit margin $p$, the holding cost $h$, and an \textit{anticipated} lead time $L$.
However, the actual lead times $\{\ell_t\}_{t=1}^T$ may differ from $L$ and are revealed only as orders arrive (or fail to arrive).

\subsection{LLM Agents} \label{sec:algorithms}

\paragraph{OR heuristic.} We first describe an OR-only baseline, which will be used by some of the LLM agents.  This baseline takes all historical data to build a per-period demand distribution (mean \& standard deviation).  It then determines a \textit{base-stock} level (i.e., inventory "target") based on this per-period distribution, the anticipated lead time $L$, and the critical fractile $\rho$ (which determines the tradeoff between under- vs.\ over- stocking).  Finally, the order quantity is decided by subtracting the total in-transit inventory from this base-stock target, and the order quantity is also \textit{capped} from above to avoid unstable inventory flow.  All in all, we end up with a data-driven \textit{capped base-stock policy} \citep{Xin2021CappedBaseStock}, which is considered a simple and powerful heuristic (see \Cref{sec:inv_mgmt}), that is fully explained in \Cref{sec:baseline}.

\paragraph{LLM agent setup.}
Next, we design LLM agents that interact with the inventory environment as sequential decision-makers.
Figure~\ref{fig:agent_architecture} illustrates the OR$\to$LLM agent architecture, which has three components: a structured system prompt, a stateless per-period decision loop, and a carry-over insight mechanism for cross-period memory.

\begin{figure}
\centering
\begin{tikzpicture}[
    >=Stealth,
    node distance=0.6cm and 0.8cm,
    box/.style={draw, rounded corners=3pt, minimum height=0.7cm, align=center, font=\small},
    sysbox/.style={box, fill=gray!12, minimum width=7.5cm, text width=7.3cm},
    inputbox/.style={box, fill=blue!8, minimum width=6.5cm},
    outputbox/.style={box, fill=green!8, minimum width=5.2cm},
    llmbox/.style={box, fill=orange!12, minimum width=14cm, minimum height=1.0cm, font=\normalsize\bfseries},
    insightbox/.style={box, fill=yellow!15, minimum width=6.5cm},
    lbl/.style={font=\scriptsize\itshape, text=gray!70},
]

\node[sysbox] (role) {Role \& objective: maximize $\sum_t (p \cdot s_t - h \cdot I_{t+1})$};
\node[sysbox, below=0.25cm of role] (mechanics) {Game mechanics: period sequence, observation delay, lead-time definition};
\node[sysbox, below=0.25cm of mechanics] (ormath) {OR baseline math \& limitations (stationary demand, promised lead time, no lost-order detection)};
\node[sysbox, below=0.25cm of ormath] (checklist) {Decision checklist \& output format};

\node[draw=gray!60, dashed, rounded corners=5pt, fit=(role)(checklist), inner sep=4pt, label={[lbl, anchor=south]above:System prompt (fixed at initialization)}] (sysprompt) {};

\node[insightbox, right=1.0cm of role] (insights) {Carry-over insights from prior periods};
\node[inputbox, below=0.25cm of insights] (obs) {Observation: $I_t$, in-transit inventory, \\ demand history, conclude msg};
\node[inputbox, below=0.25cm of obs] (orrec) {OR recommendation: $B_t$, $\mathrm{IP}_t$, $\bar{d}$, $s_d$, cap};

\node[draw=gray!60, dashed, rounded corners=5pt, fit=(insights)(orrec), inner sep=4pt, label={[lbl, anchor=south]above:User message (changes each period $t$)}] (usermsg) {};

\node[llmbox, below=1.6cm of $(checklist.south)!0.5!(orrec.south)$] (llm) {LLM (single call per period)};

\node[outputbox, below=0.6cm of llm.south west, anchor=north west, xshift=0.3cm] (action) {Order quantity $q_t$ + rationale};
\node[insightbox, below=0.6cm of llm.south east, anchor=north east, xshift=-0.3cm] (newinsight) {Updated carry-over insights};

\draw[->, thick, gray!70] (sysprompt.south) -- (sysprompt.south |- llm.north);
\draw[->, thick, gray!70] (usermsg.south) -- (usermsg.south |- llm.north);
\draw[->, thick] (llm.south -| action.north) -- (action.north);
\draw[->, thick] (llm.south -| newinsight.north) -- (newinsight.north);

\draw[->, thick, dashed, blue!50] (newinsight.east) -- ++(1.1,0) |- ([yshift=0.15cm]insights.east);
\node[lbl, blue!60] at ([xshift=1.0cm, yshift=-0.3cm]newinsight.east) {next period};

\end{tikzpicture}
\caption{Architecture of the OR$\to$LLM agent. Each period involves a single LLM call with a fixed system prompt and a period-specific user message, where the "user" here need not involve a human. The agent outputs an order quantity with rationale and optionally updates carry-over insights, which persist to the next period's input.}
\label{fig:agent_architecture}
\end{figure}
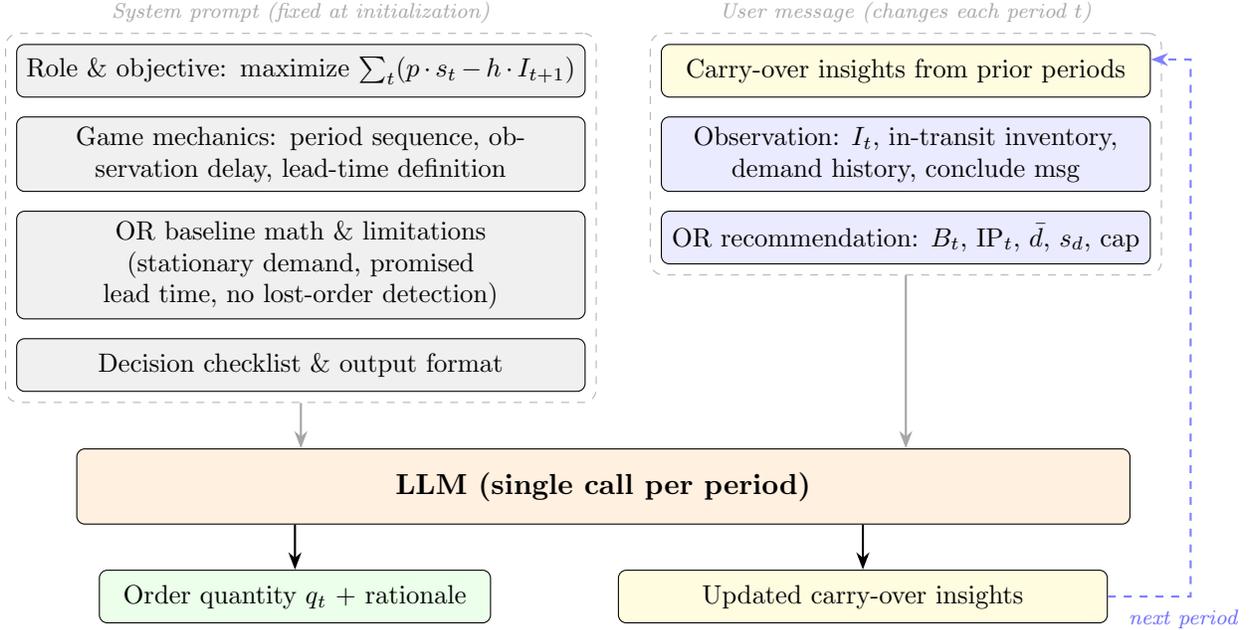

\paragraph{System prompt.}
Each agent receives a fixed system prompt describing the game mechanics (period execution sequence, observation delays, lead-time inference rules), the OR baseline's mathematical formulation, a four-step decision checklist, and the required JSON output format.
In the OR$\to$LLM mode, the prompt also explains the OR algorithm's assumptions---stationary demand, promised lead time, no lost-order detection---so the LLM knows when to override.
The complete system prompts are provided in \Cref{sec:llm_prompts}.

\paragraph{Per-period decision loop.}
Each period $t$ involves a single, stateless LLM call.
The user message contains: (1)~carry-over insights from prior periods, (2)~the current observation ($I_t$, in-transit inventory, demand history, period-conclusion messages), and (3)~the OR recommendation with statistics.
The agent returns a JSON response with the order quantity, a step-by-step rationale, a short human-readable summary, and an optional carry-over insight update.

\paragraph{Carry-over insights.}
To provide cross-period memory without maintaining a full conversation history over the $T$-period horizon, the agent can record concise memos about sustained discoveries (e.g., ``demand shifted from ${\sim}100$ to ${\sim}200$'' or ``lead time is 3 periods, not the promised 1'').
These are prepended to the next period's observation and can be updated or removed as conditions change.

\paragraph{Approaches.}
In Part~1 (Section~\ref{sec:alg_experiment}), we evaluate four methods: \textbf{OR} (data-driven capped base-stock policy), \textbf{LLM} (agent alone, no OR input), \textbf{OR$\to$LLM} (agent receives and may override OR recommendation), and \textbf{LLM$\to$OR} (inputs to capped base-stock policy are determined by LLM instead of directly by data; see \Cref{sec:llm_to_or_interface}).
All LLM-based methods share the same prompt structure shown in Figure~\ref{fig:agent_architecture}, adapted to their role: the LLM-only variant omits the OR baseline sections, while LLM$\to$OR modifies the output format to return parameter estimates instead of order quantities.
The human collaboration modes in Part~2 (Section~\ref{sec:classroom_setup}) build on the same OR$\to$LLM pipeline, adding a human who either makes the final decision (Mode~B) or provides strategic guidance (Mode~C).

\section{Part 1: Algorithmic Experiment} \label{sec:alg_experiment}

\subsection{\textsc{InventoryBench} Instances} \label{sec:instances}

An \emph{instance} provides one complete specification of the inventory control problem, defining: a finite horizon $T$; fixed realizations of all demands $\{d_t\}_{t=1}^T$ and lead times $\{\ell_t\}_{t=1}^T$; natural-language contextual information $\{x_t\}_{t=1}^T$; parameters $(p, h)$; and the initial data $\{d_{-4},\ldots,d_0\}$.
All randomness is resolved in advance to ensure full reproducibility and enable controlled comparisons across decision-making methods.

We construct two families of instances:
\begin{itemize}
\item \textbf{Synthetic instances} (720 total): designed from 10 parametric demand pattern families including stationary processes, abrupt mean shifts, gradual trends, seasonal patterns, etc.  Tests a method's ability to react to demand shifts and different patterns, which may be ad hoc. See Appendix~\ref{sec:synthetic_spec} for details.
\item \textbf{Real instances} (600 total): derived from historical sales data in the H\&M Personalized Fashion Recommendations dataset \citep{Kaggle_HM_Recommendations_2022}, using 200 distinct products across diverse categories (swimwear, knitwear, tailored clothing, basics, accessories).  Tests a method's ability to react to demand shifts and patterns that may be explainable by world knowledge, e.g.\ holidays and seasons.  See Appendix~\ref{sec:real_inst_spec} for details.
\end{itemize}

Each instance family is crossed with three lead-time configurations:
\begin{itemize}
\item $L = 0$: immediate delivery (orders arrive in the same period they are placed)
\item $L = 4$: fixed 4-period delay
\item Stochastic lead times: each order independently has realized lead time $\ell_t \in \{1, 2, 3, \infty\}$ with equal probability, where $\ell_t = \infty$ indicates a lost order
\end{itemize}

Orthogonally, each instance family is crossed with three possible $(p, h)$ values defining different critical fractiles: $\rho \in \{0.50, 0.80, 0.95\}$, representing different degrees of preference for overstocking.

\subsection{Overall results} \label{sec:metric_overall_result}

All results in the main body are reported for Gemini 3 Flash; analogous figures and tables for Grok 4.1 Fast and GPT-5 Mini appear in Appendix~\ref{sec:additional_model_results}, and detailed performance breakdowns for all three models appear in Appendix~\ref{sec:detailed_tables}.

\paragraph{Performance metric.} We refer to the total profit $\sum_{t=1}^T (p \cdot s_t - h \cdot I_{t+1})$ of a method on an instance as its \emph{reward}. To enable comparison across instances with different demand scales and parameters, we define
\[
\text{normalized reward} \;:=\; \max\!\left\{\; \frac{\text{reward}}{p \cdot \sum_{t=1}^T d_t},\; 0 \;\right\}.
\]
The denominator $p \cdot \sum_{t=1}^T d_t$ is the revenue obtained by capturing all demand while never holding any leftover inventory---a very optimistic upper bound on optimal profit. Clipping at zero prevents instances with large losses from producing extreme negative ratios (e.g., $-100$) that would dominate the average. For any set of instances, we report the mean normalized reward together with a 95\% confidence interval for the mean.

Figure~\ref{fig:gemini_overall} summarizes performance across all 1{,}320 instances, pooling real and synthetic instances.

\begin{figure}
\centering
\includegraphics[width=\textwidth]{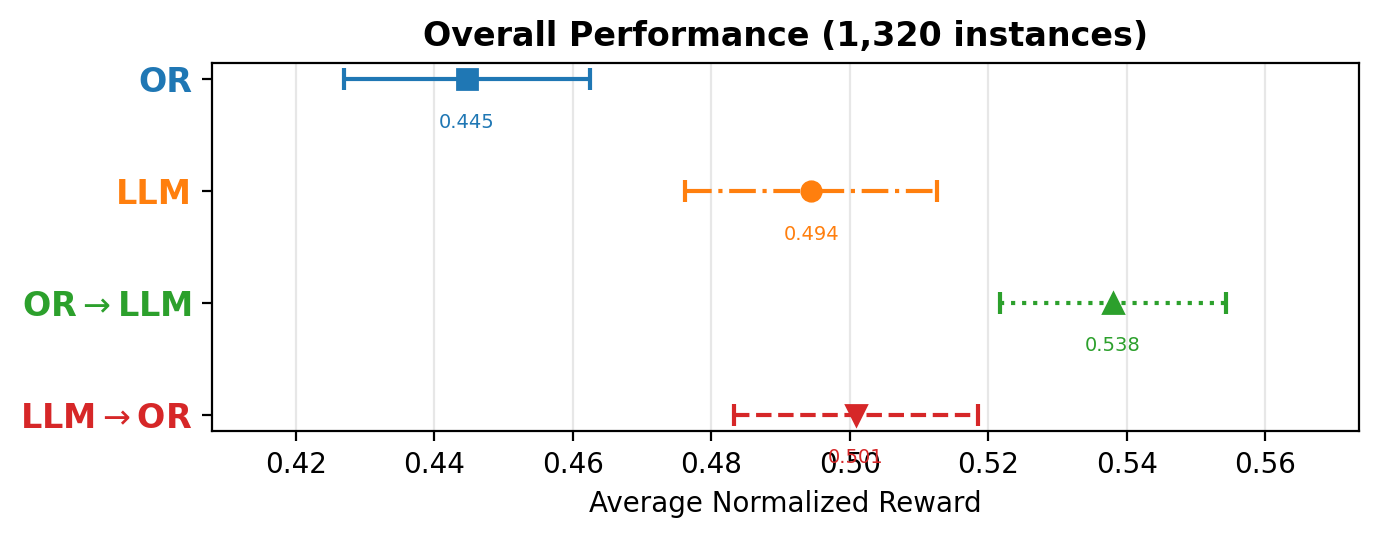}
\caption{Gemini 3 Flash: overall normalized reward (mean $\pm$ 95\% CI) across all 1{,}320 instances, by method. OR$\to$LLM achieves the highest mean (0.538), with the other hybrid method (LLM$\to$OR) coming second.}
\label{fig:gemini_overall}
\end{figure}

\subsection{Disaggregation by Lead Time} \label{sec:disagg_lead_time}

Figure~\ref{fig:gemini_by_leadtime} disaggregates performance by lead time, unveiling when different methods shine.

\begin{figure}
\centering
\includegraphics[width=\textwidth]{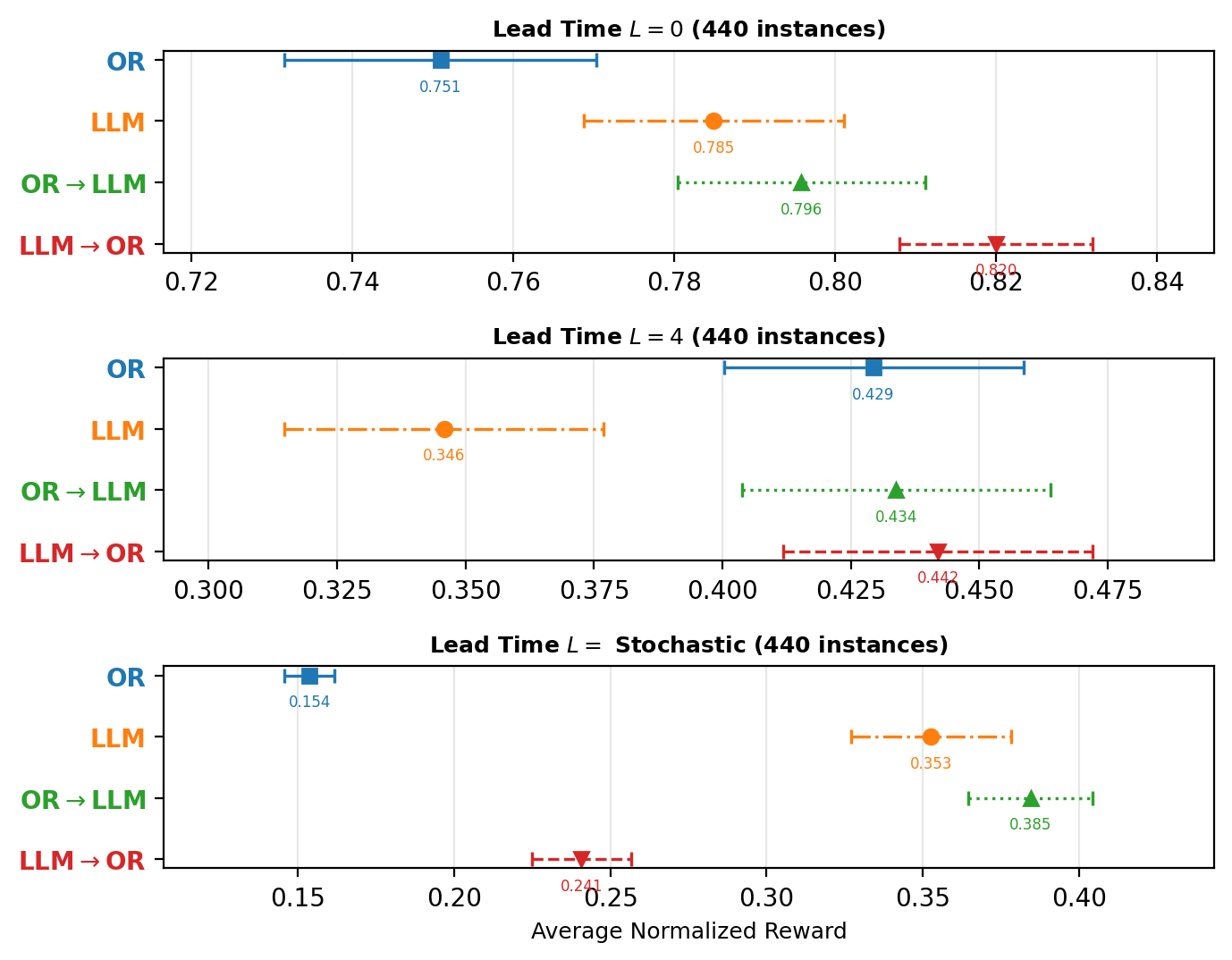}
\caption{Gemini 3 Flash: normalized reward by lead time setting (440 instances per setting).}
\label{fig:gemini_by_leadtime}
\end{figure}

Under deterministic lead times ($L = 0$ and $L = 4$), the primary value of the LLM lies in predicting demand shifts---using common-sense reasoning to detect ad hoc patterns in synthetic instances and world knowledge (e.g., seasonality) in real instances---while the final ordering decision is best left to the OR heuristic, which can translate these forecasts into precise capped base-stock calculations. This is why LLM$\to$OR is the strongest method under both settings. It really dominates when $L=0$, under which the OR heuristic is essentially perfect given a perfect prediction, so it makes sense that the LLM focuses fully on prediction, as in the LLM$\to$OR method.  $L=4$ is similar, but highlights more the downside of not having the structured calculations of the capped base-stock OR heuristic, where the LLM alone really lags behind.

The script flips under stochastic lead times. The capped base-stock heuristic is not designed for this out-of-distribution setting: it assumes a known, fixed lead time and has no mechanism to detect or adapt to lost orders. Methods in which OR computes the final order suffer accordingly (OR: 0.154, LLM$\to$OR: 0.218). By contrast, methods in which the LLM makes the final ordering decision prove far more robust (LLM: 0.353, OR$\to$LLM: 0.385), because general-purpose LLMs can reason about unexpected events---such as recognizing that an order has been lost---and adjust their behavior accordingly (see Example~\ref{ex:llm_lost_orders} for a reasoning anecdote).

\subsection{Disaggregation by Instance Family} \label{sec:disagg_instance_family}

Table~\ref{tab:gemini_by_pattern} disaggregates performance by synthetic demand pattern, restricting to deterministic lead times ($L \in \{0, 4\}$) so that differences reflect demand-forecasting ability rather than the lead time effects analyzed in \Cref{sec:disagg_lead_time}. LLM$\to$OR achieves the highest overall synthetic reward (.709) and the highest real-instance reward (.537), consistent with the aggregate findings above.

\definecolor{cBest}{RGB}{180,235,180}
\definecolor{cWorst}{RGB}{255,180,180}
\begin{table}
\centering
\caption{Gemini 3 Flash: normalized reward by synthetic demand pattern (See \Cref{sec:synthetic_spec} for definitions), each corresponding to 48 instances (8 instances for each of 3 critical fractiles and 2 lead times, with \textbf{stochastic lead time excluded}).  The \emph{All Synthetic} column averages over all 480 instances while the \emph{Real} column averages over 400 real instances (again excluding stochastic lead time). Cell shading per column: \colorbox{cBest}{green} = best, \colorbox{cWorst}{red} = worst. Best method per column in bold.}
\label{tab:gemini_by_pattern}
\resizebox{\textwidth}{!}{%
\begin{tabular}{@{}l cccccccccc |c |c@{}}
\toprule
& \rotatebox{70}{Stationary IID} & \rotatebox{70}{Mean $\uparrow$} & \rotatebox{70}{Mean $\downarrow$} & \rotatebox{70}{Trend $\uparrow$} & \rotatebox{70}{Trend $\downarrow$} & \rotatebox{70}{Variance Change} & \rotatebox{70}{Seasonal} & \rotatebox{70}{Multi Changepoint} & \rotatebox{70}{Spike/Dip} & \rotatebox{70}{Autocorr.} & \rotatebox{70}{\emph{All Synthetic}} & \rotatebox{70}{\emph{Real}} \\[-2pt]
& {\scriptsize(p01)} & {\scriptsize(p02)} & {\scriptsize(p03)} & {\scriptsize(p04)} & {\scriptsize(p05)} & {\scriptsize(p06)} & {\scriptsize(p07)} & {\scriptsize(p08)} & {\scriptsize(p09)} & {\scriptsize(p10)} & & \\
\midrule
OR & \cellcolor{cBest} \textbf{.802} & \cellcolor{cWorst!27!cBest} .774 & \cellcolor{cWorst} .557 & \cellcolor{cWorst} .626 & \cellcolor{cWorst} .521 & \cellcolor{cBest} \textbf{.706} & \cellcolor{cBest} \textbf{.676} & \cellcolor{cBest} \textbf{.700} & \cellcolor{cWorst!23!cBest} .651 & \cellcolor{cBest} \textbf{.756} & \cellcolor{cWorst!41!cBest} .677 & \cellcolor{cWorst} .486 \\
LLM & \cellcolor{cWorst} .653 & \cellcolor{cWorst} .703 & \cellcolor{cWorst!90!cBest} .567 & \cellcolor{cWorst!25!cBest} .749 & \cellcolor{cWorst!44!cBest} .597 & \cellcolor{cWorst} .583 & \cellcolor{cWorst} .585 & \cellcolor{cWorst} .628 & \cellcolor{cWorst} .595 & \cellcolor{cWorst} .637 & \cellcolor{cWorst} .630 & \cellcolor{cWorst!96!cBest} .488 \\
OR$\to$LLM & \cellcolor{cWorst!21!cBest} .771 & \cellcolor{cWorst!11!cBest} .789 & \cellcolor{cWorst!45!cBest} .611 & \cellcolor{cWorst!18!cBest} .761 & \cellcolor{cWorst!42!cBest} .600 & \cellcolor{cWorst!33!cBest} .666 & \cellcolor{cWorst!36!cBest} .644 & \cellcolor{cWorst!1!cBest} .699 & \cellcolor{cWorst!32!cBest} .644 & \cellcolor{cWorst!40!cBest} .708 & \cellcolor{cWorst!25!cBest} .689 & \cellcolor{cWorst!23!cBest} .525 \\
LLM$\to$OR & \cellcolor{cWorst!15!cBest} .780 & \cellcolor{cBest} \textbf{.800} & \cellcolor{cBest} \textbf{.656} & \cellcolor{cBest} \textbf{.791} & \cellcolor{cBest} \textbf{.657} & \cellcolor{cWorst!37!cBest} .660 & \cellcolor{cWorst!17!cBest} .661 & \cellcolor{cWorst!1!cBest} \textbf{.699} & \cellcolor{cBest} \textbf{.668} & \cellcolor{cWorst!32!cBest} .717 & \cellcolor{cBest} \textbf{.709} & \cellcolor{cBest} \textbf{.537} \\
\bottomrule
\end{tabular}%
}
\end{table}

However, OR alone is the best method for several synthetic families: Stationary IID (.802), where the OR heuristic's demand model is correctly specified, and the LLM may overfit to false patterns (see Example~\ref{ex:llm_overfits}). OR alone is also best for Variance Change (.706), Seasonal (.676), Multi Change-Point (.700), and Autocorrelated (.756), where these patterns are difficult for an LLM to anticipate.

The LLM-based methods excel precisely where changes are directional and describable: Mean Shift Up (.800 for LLM$\to$OR vs.\ .774 for OR), Mean Shift Down (.656 vs.\ .557), Trend Up (.791 vs.\ .626), and Trend Down (.657 vs.\ .521).  These families have the largest absolute gaps in the table, with LLM$\to$OR improving over OR by 3--17 percentage points.  This pattern is intuitive: an upward or downward shift in mean or trend is exactly the kind of structural break that an LLM can detect from context and translate into an updated demand forecast (see Example~\ref{ex:llm_demand_shift}).

On real H\&M instances, all LLM-involving methods outperform OR (.486), with LLM$\to$OR reaching .537.  Indeed, the LLMs leverage world knowledge to improve demand predictions beyond what the simple statistical OR heuristic can achieve (see Example~\ref{ex:llm_world_knowledge}).

\subsection{Calibration to Critical Fractiles} \label{sec:calibration_cr}

Finally, we test whether the LLM is able to change its under- vs.\ over- stocking tendencies depending on the critical fractile $\rho$ for the instance at hand.  To do so, we define each method's \emph{implicit critical fractile} as the fraction of periods in which it has excess inventory (as opposed to stocking out).  A higher implicit critical fractile implies a greater tendency to overstock, which should be desired for higher values of $\rho$ if the decision-making method is "calibrated".

We display the average implicit critical fractiles for the four methods, disaggregating instances by $\rho$.
\Cref{tab:gemini_implicit_cr} reveals a striking spectrum. The four methods can be ordered by how much LLM influence they contain---OR (none), LLM$\to$OR (LLM provides a forecast, OR makes the final order), OR$\to$LLM (OR provides a suggestion, LLM makes the final order), LLM (full LLM control)---and along this spectrum, responsiveness to~$\rho$ monotonically decreases. OR's implicit critical fractile rises from 0.470 to 0.597 as $\rho$ increases from 0.50 to 0.95 (range = 0.127); LLM$\to$OR's rises from 0.586 to 0.691 (range = 0.105); OR$\to$LLM's from 0.672 to 0.741 (range = 0.069); and LLM's from 0.736 to 0.764 (range = 0.028). Every method responds \emph{positively} to increases in~$\rho$, but the response attenuates as the final decision moves further from the OR heuristic.
We note that due to the lead times, it is impossible for the implicit critical fractile to match $\rho$ exactly, so we are just measuring the range to rank the methods.

\begin{table}
\centering
\caption{Gemini 3 Flash: implicit critical fractiles, averaged across the instances with a given value of $\rho$.}
\label{tab:gemini_implicit_cr}
\begin{tabular}{lcccc}
\toprule
\textbf{Method} & \textbf{$\rho=0.50$} & \textbf{$\rho=0.80$} & \textbf{$\rho=0.95$} & \textbf{Range} \\
\midrule
OR          & 0.470 & 0.527 & 0.597 & 0.127 \\
LLM                         & 0.736 & 0.752 & 0.764 & 0.028 \\
OR$\to$LLM                  & 0.672 & 0.712 & 0.741 & 0.069 \\
LLM$\to$OR                  & 0.586 & 0.643 & 0.691 & 0.105 \\
\bottomrule
\end{tabular}
\end{table}

\section{Part 2: Human-in-the-loop Experiment and Inventory Game}
In practice, human decision makers typically interact with OR algorithms while retaining final decision authority. As LLMs enter the loop, a natural question arises: how should these interactions be designed? Should humans provide high-level strategic guidance, or continue to make low-level operational decisions as before? More fundamentally, can human--AI collaboration generate performance gains that neither humans nor AI can achieve on their own---so that the two act as complements rather than substitutes?

We describe the experimental design and present the results, followed by a rigorous analysis of human-AI complementarity in Section~\ref{sec:complementarity_theory}.

\subsection{Game and Experiment Setup} \label{sec:classroom_setup}

We conduct a classroom experiment in which human participants play the inventory control problem from Section~\ref{sec:game_description} under the three human-AI collaboration modes introduced in Section~\ref{sec:introduction} (Figure~\ref{fig:modes}).

\begin{figure}
    \centering
    \includegraphics[width=0.85\textwidth]{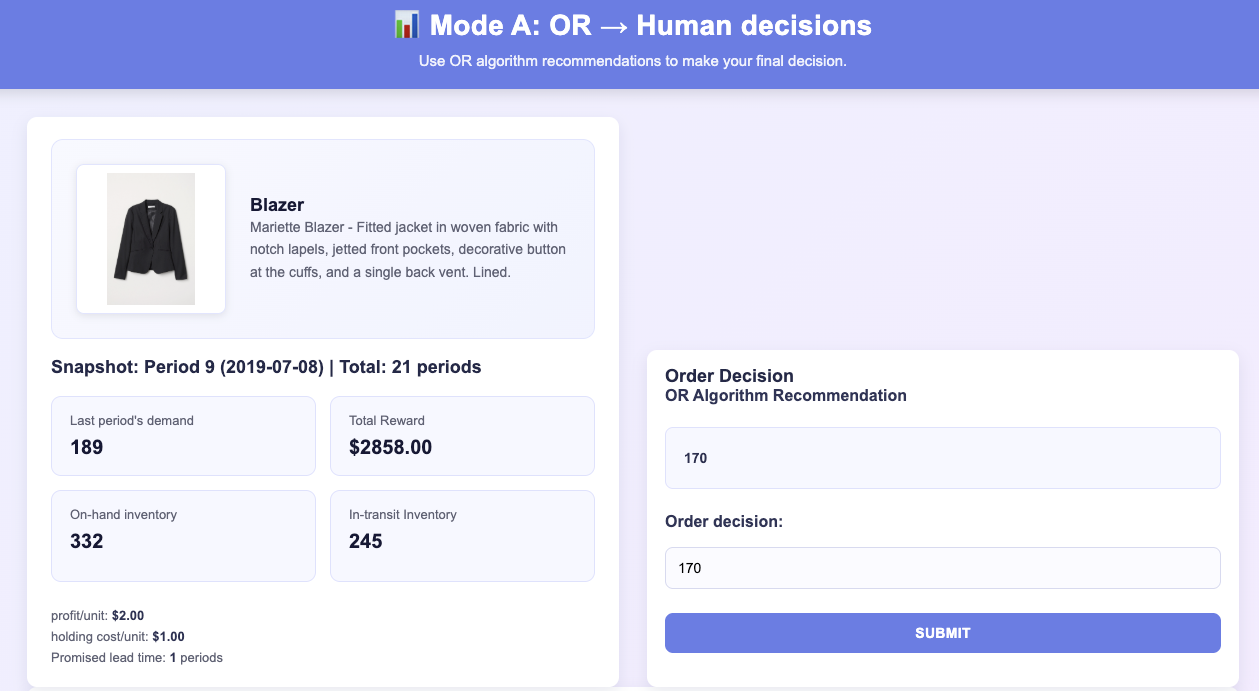}
    \caption{Decision panel of the inventory game. Each period, the participant sees the product description, current inventory state (on-hand, in-transit, cumulative reward), parameters, and the most recent demand, along with AI recommendations that vary by collaboration mode. See also Figure~\ref{fig:interface_bottom} in the appendix for the analytics panel.}
    \label{fig:game_interface}
\end{figure}

The web-based game interface (Figure~\ref{fig:game_interface}) displays all information needed for each ordering decision: the product description, current inventory state (on-hand, in-transit, cumulative reward), parameters ($p$, $h$, $L$), recent demand, and interactive charts of historical demand and inventory status.

\paragraph{Collaboration modes (treatments).}
The three modes vary how decision authority is shared:

\begin{enumerate}
    \item \textbf{Mode~A: OR $\rightarrow$ Human} (Figure~\ref{fig:modeA}).
    The participant sees the OR baseline recommendation and makes the final order decision.
    This represents the status quo where humans receive OR recommendations and retain full authority.

    \item \textbf{Mode~B: OR $\rightarrow$ LLM $\rightarrow$ Human} (Figure~\ref{fig:modeB}).
    The participant sees the OR recommendation, the LLM recommendation (from the same OR$\to$LLM pipeline as Part~1), and the LLM's written rationale, then makes the final decision.

    \item \textbf{Mode~C: OR  $\rightarrow$ LLM + Human guidance} (Figure~\ref{fig:modeC}).
    The LLM agent autonomously makes ordering decisions each period.
    The participant shifts to a strategic advisor role, providing optional free-form guidance at scheduled pauses (every four periods), which the LLM incorporates into subsequent decisions.
\end{enumerate}

\paragraph{Experiment procedure.}
Each participant plays three game instances (\Cref{sec:human_instances}), with one collaboration mode randomly assigned per instance so that each participant experiences each mode exactly once.
The complete decision trajectory---every order quantity, timestamp, and any text guidance---is recorded.
There were 69 participants and 187 subject--instance observations in total.\footnote{Some participants did not complete every instance due to time constraints or other factors.}
The game was voluntary with no identification saved or grade impact; hypotheses were pre-registered prior to the experiment.\footnote{See \href{https://aspredicted.org/gs5tm3.pdf}{https://aspredicted.org/gs5tm3.pdf}.}

\subsection{Results} \label{sec:game_results}

We first present the overall performance of five decision-making methods: three human collaboration modes (Modes A, B, and C) and two automated benchmarks (OR and OR$\rightarrow$LLM).
We consider normalized rewards as defined in \Cref{sec:metric_overall_result}.
\cref{tab:intro_1b} from \cref{sec:introduction} reports the normalized reward averaged across the three instances for each mode.
The results show that \textit{Mode~B achieves the best performance overall}, outperforming both the other human collaboration modes and the automated methods.
Moreover, the human collaboration modes outperform their non-human counterparts: Mode~A outperforms OR, and Mode~B outperforms OR$\rightarrow$LLM.
We now test whether these findings hold more rigorously using regression analysis.

\subsubsection{Comparing human collaboration modes} \label{sec:comparing_modes}

We run the following pooled OLS specification with subject fixed effects and instance fixed effects:
\[
Y_{ij} \;=\; \alpha_i \;+\; \beta_j \;+\; \tau_{c(i,j)} \;+\; \varepsilon_{ij},
\]
where $Y_{ij}$ is the normalized reward obtained by subject $i$ on instance $j$, $\alpha_i$ captures subject-level heterogeneity, $\beta_j$ captures instance difficulty, and $\tau_{c(i,j)}$ is the treatment effect for the decision-support mode used in that subject-instance pair. Standard errors are clustered at the subject level to account for within-subject correlation across instances. 
The omitted (baseline) category is Mode A on Instance 1.

\paragraph{Mode B $>$ Mode A.}
Mode B yields a statistically significant improvement over Mode A, with an estimated treatment effect $\widehat{\tau}_B = 0.0675$ (SE $=0.0185$, one-sided $p=0.00025$), which is positive and significant.
While the previous section showed that LLMs add value on top of OR algorithms, this result shows that adding LLM-based decision support on top of a traditional OR recommendation improves performance even when a human remains the final decision maker.
Since Mode A closely mirrors the status quo in many operational settings (OR provides a prescriptive recommendation and a human executes the final action), this finding suggests that incorporating an LLM layer can improve performance without removing humans from the decision process.

\paragraph{Mode B $>$ Mode C.}
Next, we compare the two modes with LLM support, Mode B and Mode C. We find a statistically significant difference between the two modes, with an estimated treatment effect $\widehat{\tau}_C - \widehat{\tau}_B = -0.0722$ (SE $=0.0177$); this difference is significant (two-sided $p=0.00012$).
This result indicates that \emph{where} the LLM is placed in the pipeline matters. In our experiment, LLMs appear to be most effective as a decision-support tool that the human can selectively follow, adjust, or override, rather than as the primary controller with only occasional human steering. 
At the same time, we acknowledge that this comparison is technology- and interface-dependent: stronger models, better mechanisms for translating human intent into concrete actions, or more frequent and structured guidance channels could narrow or potentially reverse the gap between Modes C and B.

\subsubsection{Comparing to no-human modes.} \label{sec:comparing_to_no_human_modes}

Next, we compare the human collaboration modes to the algorithm-only methods from \cref{sec:alg_experiment}. These comparisons isolate the incremental value of adding a human at the final step, holding the underlying method fixed: we compare Mode A to OR, and Mode B to OR$\rightarrow$LLM. For both methods, we evaluate the same three instances used in the human experiment. We run OR$\rightarrow$LLM 100 times per instance (to average over stochasticity in the LLM) and OR once per instance (since it is deterministic).

\paragraph{Human-in-the-loop adds value.}
We estimate regressions with instance fixed effects and an indicator for the human mode (Mode A or Mode B). Standard errors are cluster-robust, clustering human observations by subject and automated observations by run (see \Cref{sec:regression_details} for details). We conduct one-sided tests for improvement and find that the human modes outperform their non-human counterparts: \textbf{Mode A $>$ OR} ($p=0.0098$) and \textbf{Mode B $>$ OR$\rightarrow$LLM} ($p=0.000009$). 
These results indicate that human judgment adds meaningful value beyond what can be achieved by automated decision-making alone, whether based on traditional OR methods or enhanced with LLM capabilities.

\section{Human-AI Complementarity}

The previous section established that human collaboration with LLM support (Mode B) outperforms both traditional human decision-making with OR alone (Mode A) and fully automated LLM-based decisions (OR$\rightarrow$LLM). 
A key question arises: does collaboration create genuine synergy, or does it merely allow weaker decision-makers to adopt AI recommendations while stronger ones ignore them? 
In this section, we formalize this complementarity by developing metrics that distinguish between population-level and individual-level effects, and we apply them to our experiment results.

\subsection{Theoretical Framework} \label{sec:complementarity_theory}

\paragraph{Setup and notation.}
Consider a decision-making task with a fixed realization of all random variables (demand, lead times, etc.).
For a given person $i$, let:
\begin{itemize}
    \item $y_{H}^{(i)}$: the performance of person $i$ acting alone,
    \item $y_{AI}$: the performance of the AI acting alone,\footnote{
    We interpret $y_{AI}$ as a deterministic quantity representing expected AI performance (even though LLM outputs can be stochastic, its realization is not known ex-ante).}
    \item $y_{H+AI}^{(i)}$: the performance of person $i$ collaborating with AI.
\end{itemize}
We write $\bar{y}_{H} = \bE_i[y_{H}^{(i)}]$ and $\bar{y}_{H+AI} = \bE_i[y_{H+AI}^{(i)}]$ for the population averages over persons.

\paragraph{Population-level complementarity.}
Following \citet{hemmer2025complementarity}, the \emph{population-level complementarity effect} is defined as
\begin{align}
    \mathrm{CE} \;=\; \bar{y}_{H+AI} - \max\bigl(\bar{y}_{H},\, y_{AI}\bigr). \label{eq:ce}
\end{align}
When $\mathrm{CE} > 0$, the average human-AI team outperforms both the average human alone and the AI alone.

\paragraph{Individual-level complementarity.}
We are also interested in whether complementarity arises at the individual level. Define the \emph{individual-level complementarity effect} for individual $i$ as
\begin{align}
    \mathrm{CE}_i \;=\; y_{H+AI}^{(i)} - \max\bigl(y_{H}^{(i)},\, y_{AI}\bigr). \label{eq:ce_i}
\end{align}
A positive $\mathrm{CE}_i$ means that person $i$ genuinely benefits from interacting with the AI: their collaborative performance exceeds both what they could achieve alone and what the AI could achieve alone.

\paragraph{Population-level complementarity does not imply individual-level complementarity.}
Importantly, $\mathrm{CE} > 0$ does not imply the existence of any individual $i$ with $\mathrm{CE}_i > 0$. To see this, consider a population in which weak decision-makers (those with $y_{H}^{(i)} < y_{AI}$) simply adopt the AI's recommendation, achieving $y_{H+AI}^{(i)} = y_{AI}$, while strong decision-makers (those with $y_{H}^{(i)} > y_{AI}$) ignore the AI entirely, achieving $y_{H+AI}^{(i)} = y_{H}^{(i)}$. In this case, $y_{H+AI}^{(i)} = \max(y_{H}^{(i)}, y_{AI})$ for every person, so $\mathrm{CE}_i = 0$ for all~$i$. 
Yet the population average $\bar{y}_{H+AI}$ can still exceed $\max(\bar{y}_{H}, y_{AI})$. 
In such a case, population-level complementarity reflects a \emph{selection effect}, rather than true human-AI synergy in which the AI genuinely boosts individual performance beyond what either could achieve alone.

\paragraph{The missing-data challenge.}
Measuring $\mathrm{CE}_i$ directly requires observing both $y_{H}^{(i)}$ and $y_{H+AI}^{(i)}$ for the same person $i$, which is a fundamental missing-data problem: each person is observed in only one condition.\footnote{One approach in the literature is to let the human decide first, reveal the AI's recommendation, and ask whether the human would revise their decision. However, this protocol does not apply to settings with sequential, multi-step decisions, where earlier choices affect later states.}
Nevertheless, we can still estimate meaningful statistics about the distribution of $\mathrm{CE}_i$ across the population.

\paragraph{Statistical framework.}
For a person $i$ drawn from the population, define
\begin{align*}
    a(i) &\;=\; \max\bigl(y_{H}^{(i)},\, y_{AI}\bigr), \\
    b(i) &\;=\; y_{H+AI}^{(i)},
\end{align*}
and let $F_a$, $F_b$ denote their population CDFs, which can be estimated by sampling individuals in one condition each. Note that $\mathrm{CE}_i = b(i) - a(i)$.

\textbf{Average individual-level complementarity.}
The average individual-level complementarity effect is
\begin{align} \label{eq:cei_estimate}
\bE_i[\mathrm{CE}_i] \;=\; \bE[b] - \bE[a],
\end{align}
which can be estimated directly from the sample means of the two groups.

\textbf{Lower bound on the probability of positive complementarity.}
Although we cannot identify $\mathrm{CE}_i$ for any specific individual, we can provide a distribution-free lower bound on the fraction of individuals who benefit from human-AI collaboration.

\begin{theorem} \label{thm:ce_lower_bound}
Let $F_a$ and $F_b$ denote the CDFs of $a$ and $b$, respectively. Then for any $\delta \geq 0$,
\[
\bP\bigl[\mathrm{CE}_i > \delta\bigr] \;\geq\; \sup_{t \in \bR}\, \bigl(F_a(t) - F_b(t + \delta)\bigr).
\]
Moreover, this bound is tight: for any marginal distributions of $a$ and $b$, there exists a joint distribution of $(a,b)$ that achieves equality.
\end{theorem}

\begin{proof}[Proof sketch]
The lower bound follows from two observations. First, the event $\{b(i) > t+\delta\} \cap \{a(i) \leq t\}$ implies $\mathrm{CE}_i > \delta$. Second, the Fr\'{e}chet inequality gives $\bP[b(i) > t+\delta,\, a(i) \leq t] \geq (1-F_b(t+\delta)) + F_a(t) - 1 = F_a(t) - F_b(t+\delta)$. Optimizing over $t$ yields the bound. For tightness, one constructs a coupling achieving $\Pr(b-a>\delta) = \varepsilon$ by splitting the marginals into retained sub-distributions (coupled monotonically so that $b - a \leq \delta$ a.s.) and remaining mass (coupled arbitrarily); the lower bound then forces equality. See \Cref{sec:proof_complementarity} for the full proof.
\end{proof}

\begin{corollary} \label{cor:ce_positive}
Setting $\delta = 0$ in \cref{thm:ce_lower_bound} yields
$\bP\bigl[\mathrm{CE}_i > 0\bigr] \geq \sup_{t \in \bR}\,\bigl(F_a(t) - F_b(t)\bigr)$.
\end{corollary}

The bound in \cref{cor:ce_positive} has an intuitive graphical interpretation: $\bP[\mathrm{CE}_i > 0]$ is at least as large as the maximum vertical gap where $F_a$ lies above $F_b$---precisely the one-sided Kolmogorov--Smirnov statistic between the two distributions. This can be estimated from finite samples using the empirical CDFs of $a$ and $b$.

\subsection{Evaluating Complementarity} \label{sec:complementarity}

We apply the framework of the previous section to our experiment results.
We interpret ``AI'' as ``LLM'' in our experiments, and hence we use the following mapping:\footnote{We treat Mode~A (OR$\to$Human) as the ``human alone'' baseline, since it represents the classic status quo in which a human decision maker receives OR recommendations without LLM involvement. Any mode involving an LLM is considered AI collaboration.}
\begin{itemize}
    \item Human-only (H): Mode A (OR $\rightarrow$ Human) 
    \item AI-only (AI): OR $\rightarrow$ LLM 
    \item Human-AI team (H+AI): Mode B (OR $\rightarrow$ LLM $\rightarrow$ Human)
\end{itemize}

\subsubsection{Population-level complementarity.}

We evaluate the population-level complementarity effect defined in Section~\ref{sec:complementarity_theory}. 
We apply the framework separately to each of the three instances, with demand realizations and inventory arrivals held fixed across decision-making methods.
We plug in estimates of the population means $\bar{y}_H = \bE_i[y_H^{(i)}]$ and $\bar{y}_{H+AI} = \bE_i[y_{H+AI}^{(i)}]$ using sample averages over participants for each instance.
Averaging over the three instances gives
$\hat{\bar{y}}_{H+AI} = 0.5338$ and $\max\{\hat{\bar{y}}_H,\,y_{AI}\} = 0.4842$, so $$\widehat{\mathrm{CE}} = 0.0496.$$
To test whether this effect is significantly positive, we bootstrap over participants\footnote{The bootstrap is run with 10{,}000 replicates, resampled within each mode-instance cell.} and obtain a one-sided $p$-value of $0.00010$.
Thus, our results exhibit population-level human--AI complementarity: averaged across instances, the human--AI team outperforms the better of the human-only and AI-only benchmarks. In relative terms, our estimate corresponds to a $0.0496/0.4842 \approx 10.2\%$ gain over the best non-collaborative baseline.

\subsubsection{Individual-level complementarity.}
As discussed in Section~\ref{sec:complementarity_theory}, population-level complementarity does not, by itself, imply that any given individual benefits from collaborating with the AI. In principle, a positive $\mathrm{CE}$ could arise purely from a selection effect: lower-performing participants follow the LLM's recommendation and match the AI-only benchmark, while higher-performing participants ignore it and match their own human-only baseline. Under this explanation, no individual would achieve collaborative performance that exceeds the better of their standalone performance and the AI's performance.

We first estimate the average individual-level complementarity effect, $\bE_i[\mathrm{CE}_i]$, using \eqref{eq:cei_estimate}, applied separately to each instance.
Table~\ref{tab:avg_ce} reports the estimates and 95\% bootstrap confidence intervals.
Instance~1 shows a clearly positive effect whose 95\% CI excludes zero; Instances~2 and~3 are not statistically significant on their own.
The negative point estimate for Instance~2 is notable: the AI performs strongly on that instance, so 70\% of Mode~A participants score at or below the AI benchmark, making $b(i)$ a high bar to clear.

\begin{table}[h]
\centering
\begin{tabular}{lcc}
\toprule
Instance & $\hat{\bE}_i[\mathrm{CE}_i]$ & 95\% confidence interval \\
\midrule
1 & $0.048$ & $(0.004,\; 0.091)$ \\
2 & $-0.012$ & $(-0.037,\; 0.009)$ \\
3 & $0.022$ & $(-0.034,\; 0.074)$ \\
\bottomrule
\end{tabular}
\caption{Per-instance estimates of average individual-level complementarity $\bE_i[\mathrm{CE}_i]$. Bootstrap 95\% CIs use 10{,}000 replicates, resampled within each mode-instance cell.}
\label{tab:avg_ce}
\end{table}

Overall, the average individual-level complementarity effect is mixed across instances. To better understand the distribution of complementarity, we next estimate what fraction of individuals experience strictly positive complementarity.

\paragraph{Fraction who benefit from human-AI collaboration.}
We apply \cref{cor:ce_positive} separately to each instance to estimate the probability that a randomly drawn individual experiences strictly positive complementarity, $\Pr(\mathrm{CE}_i>0)$.
Since each participant is observed in only one condition, we cannot compute $\mathrm{CE}_i$ directly, but we can estimate the marginal distributions $P$ and $Q$ of $a$ and $b$ separately for each instance from the Mode B and Mode A participants assigned to it.

Table~\ref{tab:lower_bounds} reports the lower bound $\sup_{t}\,(\hat{F}_a(t) - \hat{F}_b(t))$ and 95\% bootstrap confidence intervals for each instance.
Across all three instances, individual-level complementarity is present: the lower bound is strictly positive and the 95\% confidence interval excludes zero in every case. The bounds range from 30.2\% to 60.5\%, indicating that a substantial fraction of individuals experience strictly positive complementarity.

\begin{table}[h]
\centering
\begin{tabular}{lcc}
\toprule
Instance & Lower bound on $\Pr(\mathrm{CE}_i > 0)$ & 95\% confidence interval \\
\midrule
1 & 60.5\% & (36.0\%, 80.5\%) \\
2 & 30.2\% & (12.0\%, 58.0\%) \\
3 & 35.9\% & (15.6\%, 66.0\%) \\
\bottomrule
\end{tabular}
\caption{Per-instance lower bounds on $\Pr(\mathrm{CE}_i > 0)$, estimated using \cref{cor:ce_positive}. Bootstrap 95\% CIs use 10{,}000 replicates, resampled within each mode-instance cell.}
\label{tab:lower_bounds}
\end{table}

This provides evidence that human–AI complementarity in our setting is not purely a population-level artifact driven by selection. A nontrivial fraction of participants neither simply follow the LLM's recommendation nor ignore it, but use the LLM's reasoning to make adjustments that improve upon both the human-only and AI-only benchmarks.
These lower bounds are conservative by construction; they represent the minimum fraction of beneficiaries consistent with the observed marginal distributions, and the true fractions may be substantially higher.

\subsection{How do Humans Add Value?}
\label{sec:mechanisms}

The preceding analysis established that human decision makers improve upon AI recommendations (Section~\ref{sec:comparing_modes}) and exhibit true complementarity with LLMs (Section~\ref{sec:complementarity}). We now ask: \emph{what specific capabilities enable humans to add value beyond what the LLM achieves alone?}

\paragraph{Humans and LLMs as co-pilots.}
Our evidence suggests that humans and LLMs function as \emph{co-pilots}: they add value through the same type of reasoning. Recall from Part~1 that LLMs improve upon OR algorithms through contextual, judgment-based capabilities—detecting demand regime changes, incorporating product seasonality knowledge, and identifying supply disruptions. Our experiments reveal that humans improve upon LLMs through the same channels.
Rather than contributing orthogonal skills, humans and LLMs are imperfectly correlated reasoners on a shared contextual dimension, which generates complementarity.

We support this claim using two complementary sources of evidence.
First, we study when humans override the LLM in Mode~B by examining deviations between the participant's final order and the LLM recommendation (Figure~\ref{fig:deviations_by_instance}). Second, we analyze the strategic guidance that participants provided in Mode~C (\cref{tab:human_reasoning_examples} in \Cref{sec:llm_reasoning_anecdotes}), where the LLM makes autonomous ordering decisions but humans provide high-level guidance every 4 periods. 
Despite coming from different collaboration modes, both sources of evidence converge on the same mechanisms that we describe below.

\begin{figure}
    \centering
    \includegraphics[width=0.9\textwidth]{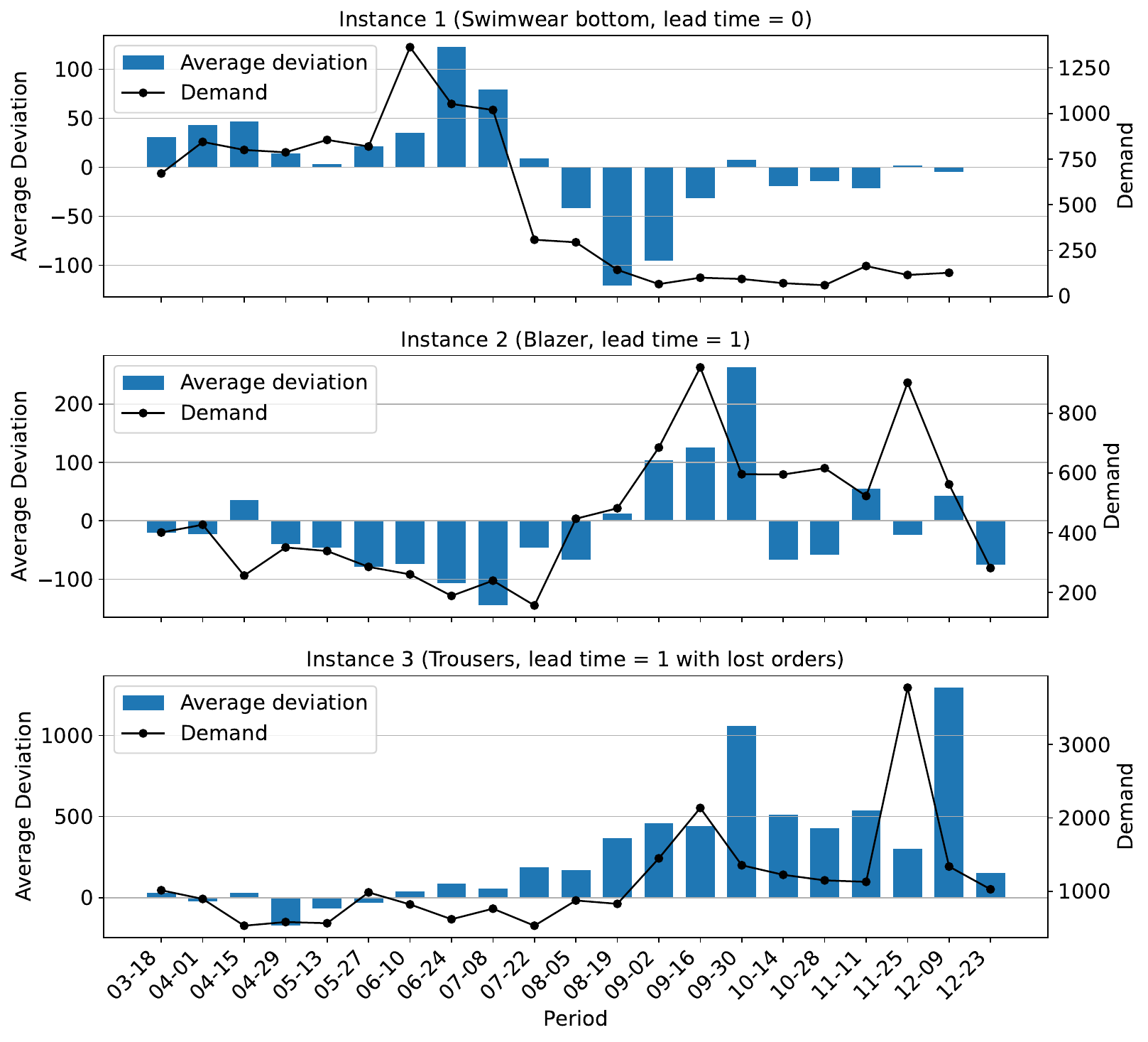}
    \caption{Average deviation of the participant's final decision from the LLM recommendation in Mode~B, by instance and period. The black line shows realized demand for each period.}
    \label{fig:deviations_by_instance}
\end{figure}

\paragraph{Reacting to demand trends and shocks.}
Across instances, we can see from Figure~\ref{fig:deviations_by_instance} that deviations often spike immediately after demand spikes, suggesting that participants intervene when recent demand realizations contain information not fully captured by the LLM's reasoning.
The first column of \cref{tab:human_reasoning_examples} (in \Cref{sec:llm_reasoning_anecdotes}) shows guidance messages in which participants explicitly direct the LLM's attention to recent demand changes.

\paragraph{Detecting anomalies: lost orders.}
Instance~3 exhibits sustained positive deviations in later periods (Figure~\ref{fig:deviations_by_instance}).
In this instance, the lead-time configuration was particularly challenging: orders would sometimes never arrive.
These ``lost orders'' were still counted as inventory-in-transit, which the OR algorithm factored into its order recommendation.
Both the LLM and the human can potentially detect such lost orders and write them off, increasing the order quantity relative to the OR recommendation.
However, the LLM does not always detect them and instead follows the OR recommendation, resulting in significant positive human deviations in the later periods of Instance~3.

\paragraph{Seasonal and world-knowledge reasoning.}
Instance~2 (blazer) features a demand pattern shaped by fashion seasonality: lower demand in summer months and higher demand heading into fall. \cref{fig:deviations_by_instance} shows that human deviations are negative during summer and positive in later periods as demand rises toward fall.
The second column of \cref{tab:human_reasoning_examples} (in \Cref{sec:llm_reasoning_anecdotes}) corroborates this with guidance messages that explicitly reference seasonality.

\paragraph{Summary.} Across all three instances, the pattern is consistent: humans add value through contextual, judgment-based reasoning—the same type of capability that makes LLMs valuable over OR. 
The co-pilot structure works because their contextual reasoning is correlated but not identical: correlated enough to agree when right, independent enough to catch each other's errors.

\section{Conclusion}

This paper provides a systematic study of how OR algorithms, LLMs, and humans can complement each other in multi-period inventory control.
On the algorithmic side, we find that combining OR and LLMs outperforms either alone, with OR$\to$LLM achieving the best overall performance on \textsc{InventoryBench}.
On the human side, our controlled experiment demonstrates that humans benefit most when they retain final decision authority with access to both the OR recommendation and the LLM's reasoning, and that humans add value beyond automated methods alone.
Our theoretical framework provides a distribution-free lower bound on the fraction of individuals who benefit from AI collaboration; applying it to our data, we find that at least 30--60\% of individuals experience strictly positive complementarity, depending on the instance.

Many directions remain open: richer agent architectures (e.g., LLMs paired with reinforcement learning or adaptive OR methods), different human collaboration structures, a broader range of participants, and extensions to other operational decision problems beyond inventory control.
We hope \textsc{InventoryBench} and the \textsc{AI-Human Inventory Game} serve as shared platforms for the community to systematically evaluate such methods going forward.

\section*{Acknowledgements}
We thank Omar Mouchtaki for pointing us to the H\&M dataset. We also acknowledge support from the Columbia AI Agents Initiative at Columbia Business School.

\bibliographystyle{abbrvnat}
\bibliography{bibliography}

\appendix
\crefalias{section}{appendix}
\crefalias{subsection}{subappendix}
\crefalias{subsubsection}{subsubappendix}

\section{OR Baseline: Data-driven Capped Base-stock Policy} \label{sec:baseline}

We consider a data-driven \textit{capped base-stock policy} \citep{Xin2021CappedBaseStock} as a default OR baseline to use for this inventory problem.  Indeed, because we have lead times and lost sales, exact optimization is intractable. The capped base-stock policy is a simple heuristic that orders up to a "target" inventory position (accounting for in-transit orders) while also implementing a cap to prevent the algorithm from ordering too much at any one time, which would cause unstable inventory flow.

At each period $t$, the data-driven version computes an order quantity as follows. Let $\hat{\mu}_t$ denote the estimated mean demand over the lead time horizon (specifically, the expected total demand from period $t$ through period $t + L$), and let $\hat{\sigma}_t$ denote the corresponding estimated standard deviation, recalling that $L$ is the anticipated lead time. These estimates are computed from the historical demand observations $\{d_{-4}, \ldots, d_0, d_1, \ldots, d_{t-1}\}$ as follows:
\begin{itemize}
    \item Let $\bar{d}$ be the empirical mean of all observed demands.
    \item Let $s_d$ be the empirical standard deviation of all observed demands.
    \item Set $\hat{\mu}_t = (1 + L) \bar{d}$ and $\hat{\sigma}_t = \sqrt{1 + L} \, s_d$, where the scaling factors $(1+L)$ and $\sqrt{1+L}$ arise from summing independent demands over $1+L$ periods \citep{snyder2019fundamentals}.
\end{itemize}

Define the \textit{critical fractile} $\rho = p / (p + h)$, which represents the optimal service level in the classical newsvendor model, and let $z^* = \Phi^{-1}(\rho)$ where $\Phi$ is the cumulative distribution function of a standard normal random variable. The \textit{base-stock level} is then $B_t = \hat{\mu}_t + z^* \hat{\sigma}_t$, which is the inventory target (for on-hand plus in-transit inventory) that balances overstocking and understocking costs.

To build intuition for readers unfamiliar with inventory theory: the critical fractile $\rho$ determines the optimal frequency of overstocking versus stocking out. Under an optimal policy, the decision maker carries leftover inventory on a $\rho$ fraction of periods and experiences a stockout on the remaining $1-\rho$ fraction. When $\rho$ is high (e.g., 0.95), stockouts are very costly relative to holding (because $p$ is high relative to $h$), so the optimal policy maintains a large safety buffer and stocks out rarely. This is generally the case, which is why Newsvendor literature tends to place asymmetric focus on high values of $\rho$.  When demand over the lead time horizon follows a normal distribution $N(\hat{\mu}_t, \hat{\sigma}_t^2)$ and lead time is zero, the base-stock level $B_t = \hat{\mu}_t + \Phi^{-1}(\rho)\,\hat{\sigma}_t$ is exactly optimal.

Let $\mathrm{IP}_t$ denote the \textit{inventory position} at time $t$, defined as on-hand inventory plus all outstanding orders:
\[
\mathrm{IP}_t \;=\; I_t + \sum_{\tau < t : \tau + \ell_\tau \ge t} q_\tau.
\]
The uncapped base-stock policy would order $q_t = \max\{B_t - \mathrm{IP}_t, 0\}$ to raise the inventory position to the target $B_t$. However, when lead times are positive, this can result in excessive orders during periods of high recent demand. To address this, we implement a \textit{capped} base-stock policy that limits each order to at most one period's worth of demand plus a safety margin:
\[
q_t \;=\; \min\left\{ \max\{B_t - \mathrm{IP}_t, 0\}, \; \frac{\hat{\mu}_t}{1+L} + \Phi^{-1}(0.95) \frac{\hat{\sigma}_t}{\sqrt{1+L}} \right\}.
\]
This cap prevents overreaction to short-term demand fluctuations while maintaining sufficient safety stock \citep{Xin2021CappedBaseStock}.  We use the 0.95 ratio as a default in the formula for the cap, irrespective of the critical fractile $\rho$.

\section{LLM Prompts} \label{sec:llm_prompts}

This appendix presents the system prompt used for the OR$\to$LLM agent. The LLM-only and LLM$\to$OR variants use the same prompt with minor modifications: the LLM-only variant omits the OR recommendation sections, and the LLM$\to$OR variant modifies the output format to return parameter estimates instead of order quantities.

Values in \texttt{\{braces\}} are filled in at initialization based on the instance configuration (e.g., product ID, promised lead time, historical demand samples).

\begin{tcolorbox}[colback=gray!5, colframe=gray!60, title={\textbf{System Prompt: Role \& Game Mechanics}}, fonttitle=\small, fontupper=\small]
\DisableQuotes\ttfamily
You control the vending machine for a single SKU "\{item\_id\}" while collaborating with an OR baseline. Maximize total reward $R_t = \text{Profit} \times \text{units\_sold} - \text{HoldingCost} \times \text{ending\_inventory}$ each period.\medskip

\textbf{Period execution sequence:}
\begin{enumerate}[leftmargin=*, nosep]
\item \textbf{VM Decision Phase:} You receive observation (including OR recommendation) and place orders for Period $N$
\item \textbf{Arrival Resolution:} Orders scheduled to arrive in Period $N$ are added to on-hand inventory
\item \textbf{Demand Resolution:} Customer demand is satisfied from on-hand inventory
\item \textbf{Period Conclusion:} System generates "Period $N$ conclude" message (visible in Period $N{+}1$)
\end{enumerate}
Important: Steps 2--4 happen AFTER your decision. You will see their results in the next period.\medskip

\textbf{Lead time definition.} Promised lead time: \{$L$\} period(s). An order placed in Period $N$ arrives during Period $(N{+}L)$'s arrival resolution, becomes visible in the ``Period $(N{+}L)$ conclude'' message, and is read at the start of Period $(N{+}L{+}1)$'s decision phase. There is always a 1-period observation delay. Actual lead time may differ from promised; orders may also be lost (never arrive).
\end{tcolorbox}

\begin{tcolorbox}[colback=blue!3, colframe=blue!40, title={\textbf{System Prompt: OR Baseline Explained to LLM}}, fonttitle=\small, fontupper=\small]
\DisableQuotes\ttfamily
The OR agent uses a capped base-stock policy:\medskip

\textbf{1. Demand estimation} (from historical samples $\xi_1, \ldots, \xi_n$):\\
\quad Empirical mean: $\bar{\mu} = \tfrac{1}{n}\sum_i \xi_i$, \quad std dev: $\bar{\sigma} = \sqrt{\tfrac{1}{n-1}\sum_i(\xi_i - \bar{\mu})^2}$\\
\quad Over lead time horizon: $\hat{\mu} = (1{+}L)\bar{\mu}$, \quad $\hat{\sigma} = \sqrt{1{+}L}\,\bar{\sigma}$\medskip

\textbf{2. Safety factor:} $q = p/(p+h)$, \quad $z^* = \Phi^{-1}(q)$\medskip

\textbf{3. Base stock:} $B = \hat{\mu} + z^* \hat{\sigma}$\medskip

\textbf{4. Capped order:} $q_t = \max\bigl(0,\, \min(B - \mathrm{IP}_t,\; \text{cap})\bigr)$\\
\quad where cap $= \hat{\mu}/(1{+}L) + \Phi^{-1}(0.95)\,\hat{\sigma}/\sqrt{1{+}L}$\medskip

\textbf{5. OR limitations:} Uses promised (not actual) lead time; weights all historical samples equally; cannot detect lost orders or regime shifts; assumes i.i.d.\ demand.\medskip

\rmfamily\textbf{Your role:} The OR recommendation is a data-driven baseline. Override it when you detect: actual vs.\ promised lead time discrepancies, demand regime changes, seasonality (from dates + product description), or lost shipments.
\end{tcolorbox}

\begin{tcolorbox}[colback=green!3, colframe=green!40, title={\textbf{System Prompt: Reasoning \& Output}}, fonttitle=\small, fontupper=\small]
\DisableQuotes\ttfamily
\textbf{Decision checklist:}
\begin{enumerate}[leftmargin=*, nosep]
\item Use world knowledge and SKU description to assess demand outlook.
\item Reconcile on-hand + pipeline with expected arrivals; flag overdue/lost shipments.
\item Inspect the OR recommendation (quantity + stats) and decide how to adapt it.
\item Justify final quantity by tying it to demand outlook, lead-time belief, and OR's baseline.
\end{enumerate}\medskip

\textbf{Carry-over insights:} Record only NEW, evidence-backed insights about sustained shifts (demand mean/variance, lead time, seasonality). Stay conservative; provide concrete stats. Remove insights once they stop being true.\medskip

\textbf{Output format} (JSON):\\
\quad \texttt{"rationale"}: full step-by-step analysis\\
\quad \texttt{"short\_rationale\_for\_human"}: 1--3 sentence summary\\
\quad \texttt{"carry\_over\_insight"}: new sustained discoveries, or ``''\\
\quad \texttt{"action"}: \{\texttt{"\{item\_id\}"}: quantity\}
\end{tcolorbox}

\paragraph{Human-in-the-loop additions.}
For Mode~B (OR$\to$LLM$\to$Human), the system prompt appends instructions for a two-stage interaction: in Stage~1, the agent provides its initial rationale and decision; in Stage~2, if the human provides feedback, the agent incorporates it and outputs a revised action.
For Mode~C (autonomous with strategic guidance), the prompt instructs the agent to follow any strategic guidance from the human supervisor that appears at the top of the observation.

\paragraph{Per-period user message.}
Each period, the user message provided to the LLM contains three blocks:
(1)~carry-over insights from prior periods (if any), formatted as a header listing each insight with its originating period;
(2)~the current observation from the game environment (on-hand inventory, in-transit orders, the previous period's conclusion message reporting arrivals, demand, sales, and ending inventory);
and (3)~the OR algorithm's recommendation with statistics (base-stock level, current inventory position, empirical demand mean and standard deviation, and order cap).

\section{LLM $\rightarrow$ OR Parameter Interface} \label{sec:llm_to_or_interface}

In the LLM $\rightarrow$ OR method, the LLM agent analyzes contextual information, demand history, and observed arrival patterns to provide parameter estimates that inform the OR algorithm's calculations. The OR algorithm then uses these LLM-provided estimates to compute the order quantity according to the capped base-stock policy described in Section~\ref{sec:game_description}.

The LLM can provide estimates for the following parameters:

\paragraph{Lead time ($L$).}
The LLM provides an estimate of the effective lead time for the current order. This estimate may differ from the anticipated lead time if the LLM detects supply disruptions, seasonal delays, or patterns in past delivery performance. For example, if recent orders have consistently arrived late or been lost, the LLM may recommend using a longer effective lead time or adjusting the demand forecast accordingly.

\paragraph{Mean demand over lead time horizon ($\hat{\mu}_t$).}
The LLM provides a forecast of the expected total demand from period $t$ through the period when the current order is expected to arrive (typically period $t + L$). This forecast may incorporate:
\begin{itemize}
    \item Pattern detection in recent demand (trends, changepoints, seasonality)
    \item Product-specific knowledge (e.g., swimwear demand peaks in spring/summer)
    \item Calendar effects (holidays, promotional periods)
    \item Context from the text description $x_t$
\end{itemize}
The LLM's forecast replaces or augments the simple empirical mean used by the default OR algorithm, allowing for adaptive forecasting that responds to non-stationary demand patterns.

\paragraph{Standard deviation over lead time horizon ($\hat{\sigma}_t$).}
The LLM provides an estimate of demand variability over the lead time horizon. This estimate may reflect recent changes in demand volatility or structural shifts in the demand process. For instance, if demand has become more erratic in recent periods, the LLM may recommend a higher $\hat{\sigma}_t$ to increase safety stock.

\paragraph{Implementation note.}
The LLM may choose to specify only $L$, in which case $\hat{\mu}_t$ and $\hat{\sigma}_t$ are computed using the OR baseline's default empirical estimates: $\hat{\mu}_t = (1+L)\bar{d}$ and $\hat{\sigma}_t = \sqrt{1+L} \, s_d$, where $\bar{d}$ and $s_d$ are the sample mean and standard deviation of observed demand. Alternatively, if the LLM provides $\hat{\mu}_t$ and $\hat{\sigma}_t$ directly, then the lead time estimate $L$ serves only as contextual information and is not used in the base-stock calculation. This allows the LLM to bypass the assumption of stationary, IID demand implicit in the default formulas.

Once the LLM provides these parameter estimates, the OR algorithm computes the base-stock level $B_t = \hat{\mu}_t + z^* \hat{\sigma}_t$ (where $z^* = \Phi^{-1}(\rho)$), determines inventory position $\mathrm{IP}_t$, and applies the capped base-stock policy to determine the order quantity $q_t$.

\section{Synthetic Instance Specifications} \label{sec:synthetic_spec}

This appendix provides the complete distributional specifications for all 40 synthetic demand variants used in the benchmark.
We generate 10 demand pattern families, each with 4 parametric variants, yielding 40 unique distributions. For each, we create 2 independent demand realizations, cross them with 3 cost ratios ($p{:}h = 1{:}1$, $4{:}1$, $19{:}1$), and evaluate under 3 lead time settings ($L=0$, $L=4$, stochastic), producing $40 \times 2 \times 3 \times 3 = 720$ synthetic instances in total. Each instance consists of a $T=50$ period test trajectory and 5 historical training samples. The 10 demand pattern families are: (1) stationary IID, (2--3) abrupt mean shifts (increase/decrease at period 16), (4--5) gradual trends (increasing/decreasing), (6) variance changes, (7) seasonal/cyclical patterns, (8) multiple changepoints, (9) temporary spikes or dips, and (10) autocorrelated AR(1) processes.

Several of these patterns feature ad hoc demand shifts of the kind commonly encountered in operations management games---for example, the step-function demand increase in the Beer Game \citep{sterman1989modeling} or the non-stationary demand environment in the Mexico-China sourcing game \citep{AllonVanMieghem2010MexicoChina}---while others feature more predictable seasonality or gradual trends. The OR baseline described in \Cref{sec:baseline} treats demand as stationary and IID, so it cannot detect any of these patterns. One could, in principle, augment the OR method with statistical changepoint detection, trend estimation, or models for correlated processes, but each extension requires its own methodology and careful tuning---there is no single off-the-shelf OR procedure that handles all 10 pattern families. By contrast, an LLM can attempt pattern recognition out of the box, without case-specific statistical machinery. Example~\ref{ex:llm_demand_shift} illustrates this capability with an example in which the LLM detects a mean shift, confirms the regime change over multiple periods, and adapts its forecast accordingly.

\subsection*{Generation Procedure}

Demand realizations are generated as follows. For normal distributions, $D_t \sim N(\mu(t), \sigma(t))$, truncated at zero and rounded to the nearest integer. For uniform distributions, $D_t \sim \text{Uniform}[a, b]$, similarly truncated and rounded. For AR(1) processes, $D_t = \varphi \cdot D_{t-1} + c + \varepsilon_t$ where $\varepsilon_t \sim N(0, \sigma)$ and $D_0 = 100$. Changepoint patterns apply the appropriate segment's distribution based on period $t$.

Training data generation depends on the pattern type. For stationary patterns (p01), training consists of 5 i.i.d.\ samples from the distribution. For changepoint patterns (p02, p03, p06, p08, p09), training consists of 5 i.i.d.\ samples from the \emph{first segment only}---so the algorithm has no advance notice of the regime change. For trend (p04, p05), seasonal (p07), and AR(1) patterns (p10), training consists of sequential samples at $t = 1, 2, 3, 4, 5$. All random number generation uses base seed 42 for reproducibility.

\subsection*{Pattern Specifications}

\paragraph{Pattern 1: Stationary IID (p01).}
Constant distribution throughout all periods.

\begin{center}
\small
\begin{tabular}{@{}llll@{}}
\toprule
\textbf{Variant} & \textbf{Distribution} & \textbf{Parameters} \\
\midrule
v1 & Normal & $\mu = 100$, $\sigma = 25$ \\
v2 & Normal & $\mu = 100$, $\sigma = 40$ \\
v3 & Normal & $\mu = 100$, $\sigma = 15$ \\
v4 & Uniform & $a = 50$, $b = 150$ \\
\bottomrule
\end{tabular}
\end{center}

\paragraph{Pattern 2: Mean Increase at $t=16$ (p02).}
Sudden increase in mean demand at period 16.

\begin{center}
\small
\begin{tabular}{@{}lllll@{}}
\toprule
\textbf{Variant} & \textbf{Before} ($t \le 15$) & \textbf{After} ($t \ge 16$) & \textbf{Change} \\
\midrule
v1 & $N(100, 25)$ & $N(200, 35)$ & $+100\%$ mean \\
v2 & $N(100, 25)$ & $N(150, 30)$ & $+50\%$ mean \\
v3 & $N(100, 25)$ & $N(300, 50)$ & $+200\%$ mean \\
v4 & $N(100, 25)$ & $N(200, 25)$ & $+100\%$ mean, same $\sigma$ \\
\bottomrule
\end{tabular}
\end{center}

\paragraph{Pattern 3: Mean Decrease at $t=16$ (p03).}
Sudden decrease in mean demand at period 16.

\begin{center}
\small
\begin{tabular}{@{}lllll@{}}
\toprule
\textbf{Variant} & \textbf{Before} ($t \le 15$) & \textbf{After} ($t \ge 16$) & \textbf{Change} \\
\midrule
v1 & $N(100, 25)$ & $N(50, 18)$ & $-50\%$ mean \\
v2 & $N(100, 25)$ & $N(70, 20)$ & $-30\%$ mean \\
v3 & $N(100, 25)$ & $N(30, 15)$ & $-70\%$ mean \\
v4 & $N(150, 30)$ & $N(80, 22)$ & $-47\%$ mean \\
\bottomrule
\end{tabular}
\end{center}

\paragraph{Pattern 4: Increasing Trend (p04).}
Gradual increase in demand over time.

\begin{center}
\small
\begin{tabular}{@{}llll@{}}
\toprule
\textbf{Variant} & $\boldsymbol{\mu(t)}$ & $\boldsymbol{\sigma(t)}$ \\
\midrule
v1 & $100t$ & $25\sqrt{t}$ \\
v2 & $50 + 3t$ & $20$ \\
v3 & $100 \times 1.05^t$ & $25$ \\
v4 & $100 + 2t$ & $25\sqrt{t}$ \\
\bottomrule
\end{tabular}
\end{center}

\paragraph{Pattern 5: Decreasing Trend (p05).}
Gradual decrease in demand over time.

\begin{center}
\small
\begin{tabular}{@{}llll@{}}
\toprule
\textbf{Variant} & $\boldsymbol{\mu(t)}$ & $\boldsymbol{\sigma(t)}$ \\
\midrule
v1 & $\max(200 - 3t,\; 50)$ & $25$ \\
v2 & $200 \times 0.97^t$ & $20$ \\
v3 & $\max(150 - 2t,\; 30)$ & $20$ \\
v4 & $200 / \sqrt{t}$ & $15$ \\
\bottomrule
\end{tabular}
\end{center}

\paragraph{Pattern 6: Variance Change at $t=16$ (p06).}
Change in demand variability at period 16.

\begin{center}
\small
\begin{tabular}{@{}lllll@{}}
\toprule
\textbf{Variant} & \textbf{Before} ($t \le 15$) & \textbf{After} ($t \ge 16$) & \textbf{Change} \\
\midrule
v1 & $N(100, 25)$ & $\text{Uniform}[0, 200]$ & Normal $\to$ Uniform \\
v2 & $N(100, 25)$ & $N(100, 50)$ & $\sigma$ doubles \\
v3 & $N(100, 50)$ & $N(100, 20)$ & $\sigma$ decreases \\
v4 & $\text{Uniform}[50, 150]$ & $N(100, 15)$ & Uniform $\to$ Normal \\
\bottomrule
\end{tabular}
\end{center}

\paragraph{Pattern 7: Seasonal/Cyclical (p07).}
Periodic demand with sinusoidal variation. All variants use $\sigma = 25$.

\begin{center}
\small
\begin{tabular}{@{}lllll@{}}
\toprule
\textbf{Variant} & $\boldsymbol{\mu(t)}$ & \textbf{Period} & \textbf{Amplitude} \\
\midrule
v1 & $100 + 30\sin(2\pi t/10)$ & 10 & 30 \\
v2 & $100 + 50\sin(2\pi t/5)$ & 5 & 50 \\
v3 & $100 + 40\sin(2\pi t/25)$ & 25 & 40 \\
v4 & $100 \times (1 + 0.3\sin(2\pi t/10))$ & 10 & 30\% mult.\ \\
\bottomrule
\end{tabular}
\end{center}

\paragraph{Pattern 8: Multiple Changepoints (p08).}
Two changepoints creating three distinct demand regimes.

\begin{center}
\small
\begin{tabular}{@{}llllll@{}}
\toprule
\textbf{Variant} & $t \in [1,15]$ & $t \in [16,35]$ & $t \in [36,50]$ & \textbf{Pattern} \\
\midrule
v1 & $N(100, 25)$ & $N(150, 30)$ & $N(80, 20)$ & Up then down \\
v2 & $N(100, 25)$ & $N(60, 20)$ & $N(140, 30)$ & Down then up \\
v3 & $N(100, 25)$ & $N(100, 50)$ & $N(100, 20)$ & Variance only \\
v4 & $N(80, 20)$ & $N(120, 25)$ & $N(100, 22)$ & Mild fluctuation \\
\bottomrule
\end{tabular}
\end{center}

\paragraph{Pattern 9: Temporary Spike/Dip (p09).}
Temporary demand anomaly followed by return to baseline.

\begin{center}
\small
\begin{tabular}{@{}llllll@{}}
\toprule
\textbf{Variant} & $t \in [1,15]$ & $t \in [16,25]$ & $t \in [26,50]$ & \textbf{Pattern} \\
\midrule
v1 & $N(100, 25)$ & $N(200, 35)$ & $N(100, 25)$ & Temporary surge \\
v2 & $N(100, 25)$ & $N(50, 18)$ & $N(100, 25)$ & Temporary dip \\
v3 & $N(100, 25)$ & $N(250, 40)$ & $N(120, 28)$ & Surge $\to$ new normal \\
v4 & $N(100, 25)$ & $N(40, 15)$ & $N(80, 22)$ & Dip $\to$ partial recovery \\
\bottomrule
\end{tabular}
\end{center}

\paragraph{Pattern 10: Autocorrelated AR(1) (p10).}
Demand follows $D_t = \varphi \cdot D_{t-1} + c + \varepsilon_t$, with $\varepsilon_t \sim N(0, \sigma)$ and $D_0 = 100$. The long-run mean is $c/(1 - \varphi) = 100$ for all variants.

\begin{center}
\small
\begin{tabular}{@{}lllll@{}}
\toprule
\textbf{Variant} & $\boldsymbol{\varphi}$ & $\boldsymbol{c}$ & $\boldsymbol{\sigma}$ & \textbf{Behavior} \\
\midrule
v1 & 0.7 & 30 & 20 & Strong positive autocorrelation \\
v2 & 0.5 & 50 & 25 & Moderate positive \\
v3 & 0.3 & 70 & 30 & Weak positive \\
v4 & $-0.3$ & 130 & 25 & Negative (alternating) \\
\bottomrule
\end{tabular}
\end{center}

\section{Real Instance Specifications} \label{sec:real_inst_spec}

The real portion of our benchmark comprises 600 instances derived from the H\&M Personalized Fashion Recommendations dataset \citep{Kaggle_HM_Recommendations_2022}.
We selected 200 distinct articles (SKUs), each with weekly aggregated sales data spanning 52 weeks: the first 5 weeks serve as historical training data and the remaining 47 weeks form the test trajectory ($T=47$). Each article is evaluated under the same 3 lead time settings (0, 4, and stochastic), and cost ratios are uniformly randomly assigned (with fixed seed) corresponding to $\rho \in \{0.50, 0.80, 0.95\}$ ($p{:}h = 1{:}1$, $4{:}1$, and $19{:}1$), distributed evenly, producing $200 \times 3 = 600$ instances.

\subsection*{Data Preprocessing}

We preprocessed the raw H\&M transaction data as follows. We aggregated individual transactions into weekly sales counts per article for all of 2019, then applied two filters: (i) articles must have positive sales in at least 50 of 52 weeks (eliminating items with frequent stockouts or discontinuations), and (ii) articles must have stable prices, defined as a max-to-min weekly average price ratio of at most 1.2 after excluding up to 4 outlier holiday weeks (eliminating items whose demand is confounded by heavy discounting or price changes). We assume that observed sales is the true demand. From the articles passing both filters, we selected the top 200 by total annual sales volume. For each article, the first 5 weeks (2019-01-07 through 2019-02-04) form the training set and the remaining weeks (2019-02-11 through 2019-12-30) form the test trajectory. Each test instance also includes a text product description drawn from the H\&M article metadata (product name, type, color, garment group, and a natural-language detail description), which the LLM can read but the OR algorithm ignores. The contextual information for each week also includes the actual real dates.

\subsection*{LLM World Knowledge}

A key advantage of LLM-based methods on real instances is their ability to leverage world knowledge about product seasonality. Example~\ref{ex:llm_world_knowledge} provides an illustrative example in which the LLM reasons about swimwear demand using calendar dates and seasonal knowledge to adjust its forecasts throughout the year.

\section{Additional Model Results} \label{sec:additional_model_results}

This appendix presents the full results for Grok 4.1 Fast and GPT-5 Mini, analogous to the Gemini 3 Flash results in \Cref{sec:alg_experiment}.

The qualitative findings for Grok 4.1 Fast are broadly consistent with Gemini. OR$\to$LLM is the best method overall (.514), followed by LLM$\to$OR (.493), mirroring Gemini's overall ranking. The same structural patterns hold across lead time settings: LLM$\to$OR dominates under deterministic lead times, while OR$\to$LLM excels under stochastic lead times where the LLM's ability to detect lost orders is decisive. On the pattern table (which excludes stochastic lead times), LLM$\to$OR achieves the highest synthetic (.707) and real (.539) rewards; OR alone remains best on families where its stationary demand model is well-specified (Stationary IID, Variance Change, Seasonal, Multi Changepoint, Autocorrelated), while LLM-based methods excel on directional demand shifts (Mean Shift Up/Down, Trend Up/Down). The calibration spectrum also mirrors Gemini: responsiveness to $\rho$ decreases with more LLM influence, with ranges of 0.127 (OR), 0.101 (LLM$\to$OR), 0.078 (LLM), and 0.070 (OR$\to$LLM).

GPT-5 Mini is the weakest of the three models overall, with OR$\to$LLM (.461), LLM (.460), and LLM$\to$OR (.458) all performing similarly. The near-parity suggests that GPT-5 Mini tends to ignore the OR recommendation, so that incorporating OR adds little. This is also reflected in its calibration: the LLM and OR$\to$LLM ranges are only 0.037 and 0.042, respectively, indicating near-constant stocking behavior that largely disregards the critical fractile.

\subsection{Grok 4.1 Fast}

\begin{figure}[H]
\centering
\includegraphics[width=\textwidth]{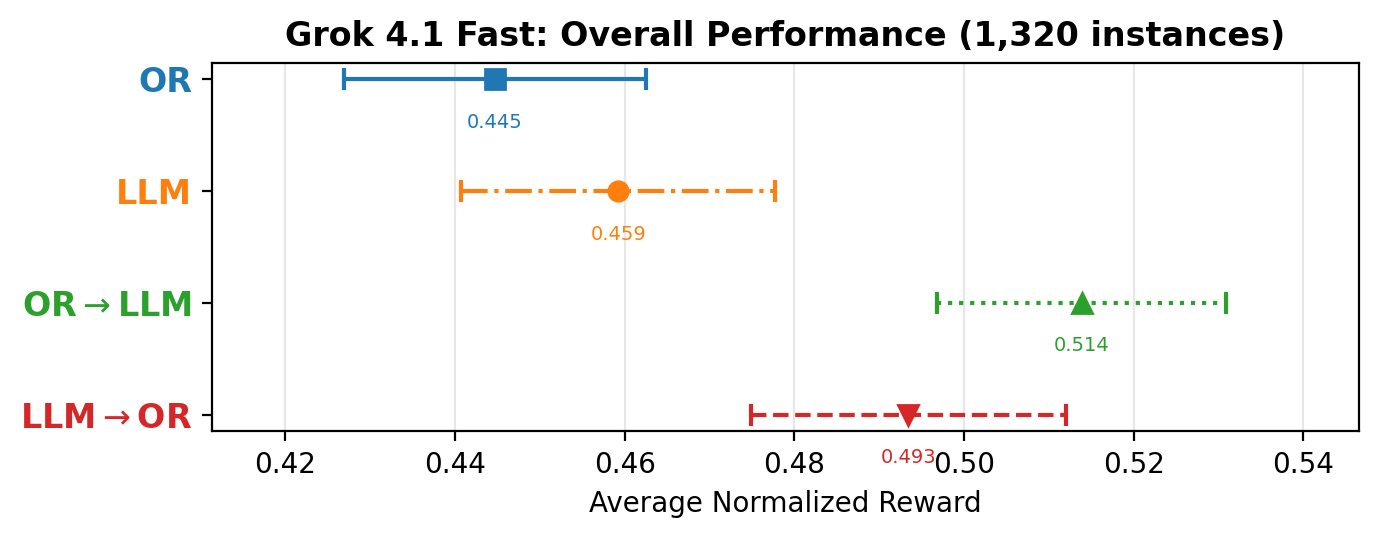}
\caption{Grok 4.1 Fast: overall normalized reward (mean $\pm$ 95\% CI) across all 1{,}320 instances, by method.}
\label{fig:grok_overall}
\end{figure}

\begin{figure}[H]
\centering
\includegraphics[width=\textwidth]{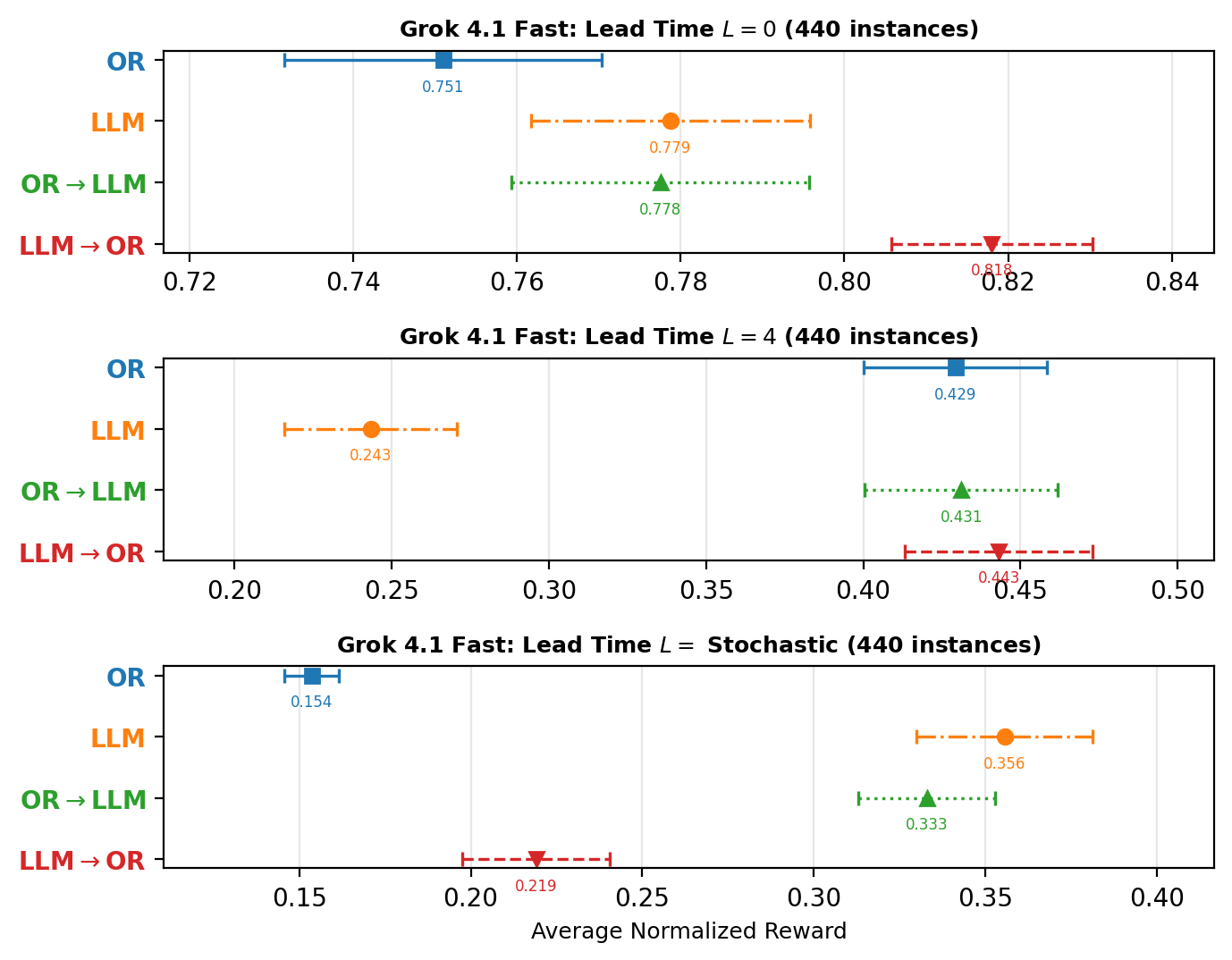}
\caption{Grok 4.1 Fast: normalized reward by lead time setting (440 instances per setting).}
\label{fig:grok_by_leadtime}
\end{figure}

\begin{table}[H]
\centering
\caption{Grok 4.1 Fast: normalized reward by synthetic demand pattern (See \Cref{sec:synthetic_spec} for definitions), each corresponding to 48 instances (8 instances for each of 3 critical fractiles and 2 lead times, with \textbf{stochastic lead time excluded}).  The \emph{All Synthetic} column averages over all 480 instances while the \emph{Real} column averages over 400 real instances (again excluding stochastic lead time). Cell shading per column: \colorbox{cBest}{green} = best, \colorbox{cWorst}{red} = worst. Best method per column in bold.}
\label{tab:grok_by_pattern}
\resizebox{\textwidth}{!}{%
\begin{tabular}{@{}l cccccccccc |c |c@{}}
\toprule
& \rotatebox{70}{Stationary IID} & \rotatebox{70}{Mean $\uparrow$} & \rotatebox{70}{Mean $\downarrow$} & \rotatebox{70}{Trend $\uparrow$} & \rotatebox{70}{Trend $\downarrow$} & \rotatebox{70}{Variance Change} & \rotatebox{70}{Seasonal} & \rotatebox{70}{Multi Changepoint} & \rotatebox{70}{Spike/Dip} & \rotatebox{70}{Autocorr.} & \rotatebox{70}{\emph{All Synthetic}} & \rotatebox{70}{\emph{Real}} \\[-2pt]
& {\scriptsize(p01)} & {\scriptsize(p02)} & {\scriptsize(p03)} & {\scriptsize(p04)} & {\scriptsize(p05)} & {\scriptsize(p06)} & {\scriptsize(p07)} & {\scriptsize(p08)} & {\scriptsize(p09)} & {\scriptsize(p10)} & & \\
\midrule
OR & \cellcolor{cBest} \textbf{.802} & \cellcolor{cWorst!27!cBest} .774 & \cellcolor{cWorst!60!cBest} .557 & \cellcolor{cWorst} .626 & \cellcolor{cWorst} .521 & \cellcolor{cBest} \textbf{.706} & \cellcolor{cBest} \textbf{.676} & \cellcolor{cBest} \textbf{.700} & \cellcolor{cWorst!9!cBest} .651 & \cellcolor{cBest} \textbf{.756} & \cellcolor{cWorst!23!cBest} .677 & \cellcolor{cWorst!49!cBest} .486 \\
LLM & \cellcolor{cWorst} .607 & \cellcolor{cWorst} .686 & \cellcolor{cWorst} .493 & \cellcolor{cWorst!57!cBest} .697 & \cellcolor{cWorst!99!cBest} .522 & \cellcolor{cWorst} .512 & \cellcolor{cWorst} .543 & \cellcolor{cWorst} .573 & \cellcolor{cWorst} .536 & \cellcolor{cWorst} .610 & \cellcolor{cWorst} .578 & \cellcolor{cWorst} .431 \\
OR$\to$LLM & \cellcolor{cWorst!6!cBest} .790 & \cellcolor{cWorst!15!cBest} .788 & \cellcolor{cWorst!32!cBest} .602 & \cellcolor{cBest} \textbf{.789} & \cellcolor{cWorst!42!cBest} .604 & \cellcolor{cWorst!14!cBest} .678 & \cellcolor{cWorst!9!cBest} .665 & \cellcolor{cWorst!8!cBest} .689 & \cellcolor{cWorst!12!cBest} .647 & \cellcolor{cWorst!17!cBest} .732 & \cellcolor{cWorst!7!cBest} .698 & \cellcolor{cWorst!44!cBest} .492 \\
LLM$\to$OR & \cellcolor{cWorst!12!cBest} .779 & \cellcolor{cBest} \textbf{.806} & \cellcolor{cBest} \textbf{.653} & \cellcolor{cWorst!12!cBest} .770 & \cellcolor{cBest} \textbf{.665} & \cellcolor{cWorst!20!cBest} .666 & \cellcolor{cWorst!16!cBest} .654 & \cellcolor{cWorst!3!cBest} .696 & \cellcolor{cBest} \textbf{.662} & \cellcolor{cWorst!25!cBest} .719 & \cellcolor{cBest} \textbf{.707} & \cellcolor{cBest} \textbf{.539} \\
\bottomrule
\end{tabular}%
}
\end{table}

\begin{table}[H]
\centering
\caption{Grok 4.1 Fast: implicit critical fractiles, averaged across the instances with a given value of $\rho$.}
\label{tab:grok_implicit_cr}
\begin{tabular}{lcccc}
\toprule
\textbf{Method} & \textbf{$\rho=0.50$} & \textbf{$\rho=0.80$} & \textbf{$\rho=0.95$} & \textbf{Range} \\
\midrule
OR          & 0.470 & 0.527 & 0.597 & 0.127 \\
LLM                         & 0.747 & 0.796 & 0.825 & 0.078 \\
OR$\to$LLM                  & 0.672 & 0.709 & 0.742 & 0.070 \\
LLM$\to$OR                  & 0.643 & 0.700 & 0.744 & 0.101 \\
\bottomrule
\end{tabular}
\end{table}

\subsection{GPT-5 Mini}

\begin{figure}[H]
\centering
\includegraphics[width=\textwidth]{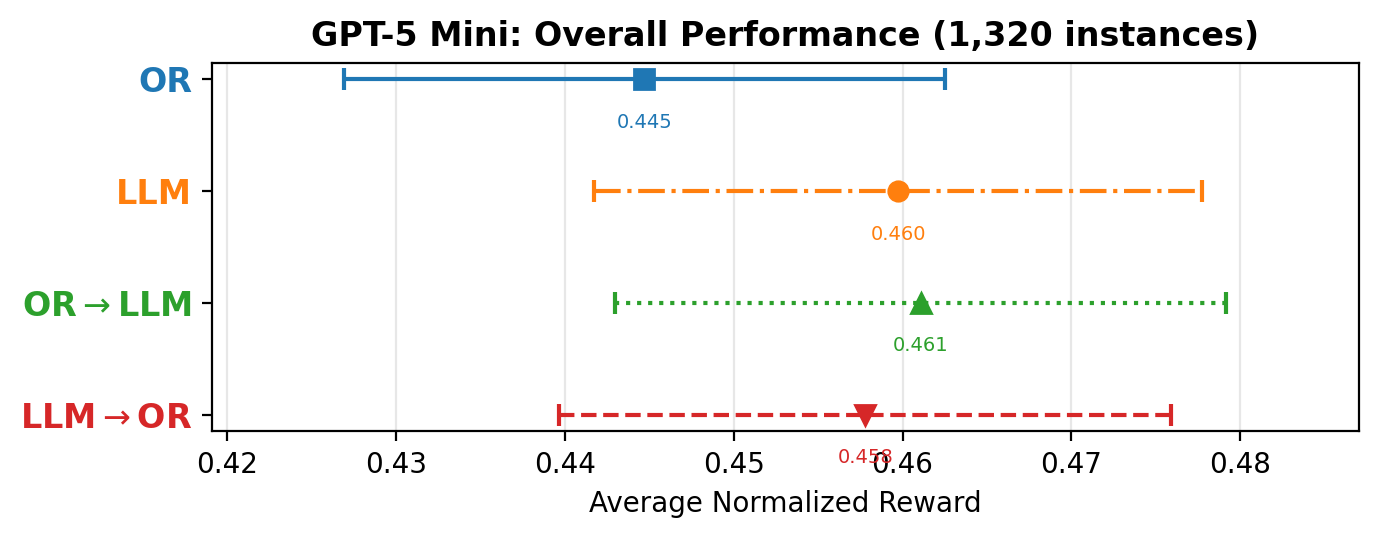}
\caption{GPT-5 Mini: overall normalized reward (mean $\pm$ 95\% CI) across all 1{,}320 instances, by method.}
\label{fig:gpt5m_overall}
\end{figure}

\begin{figure}[H]
\centering
\includegraphics[width=\textwidth]{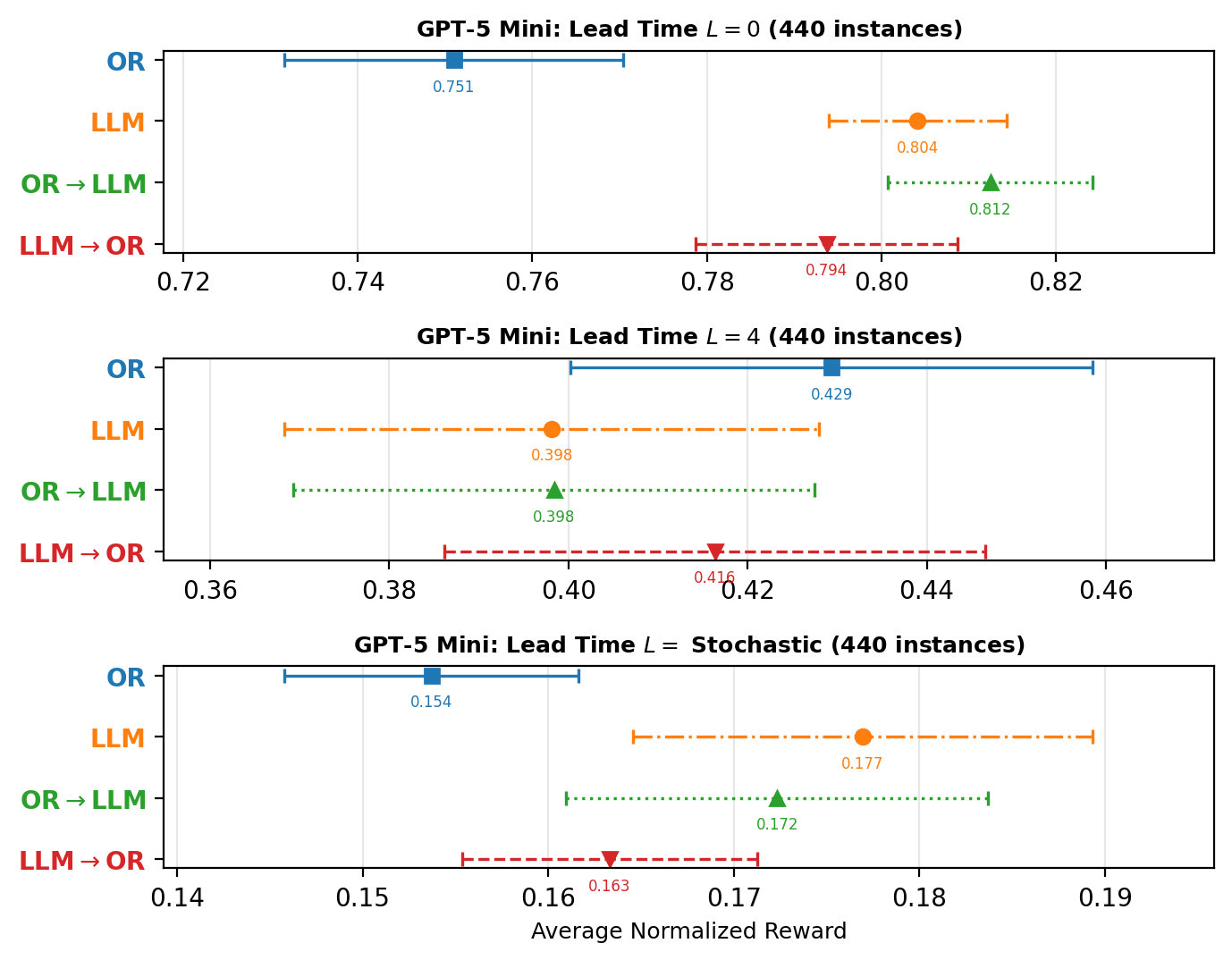}
\caption{GPT-5 Mini: normalized reward by lead time setting (440 instances per setting).}
\label{fig:gpt5m_by_leadtime}
\end{figure}

\begin{table}[H]
\centering
\caption{GPT-5 Mini: normalized reward by synthetic demand pattern (See \Cref{sec:synthetic_spec} for definitions), each corresponding to 48 instances (8 instances for each of 3 critical fractiles and 2 lead times, with \textbf{stochastic lead time excluded}).  The \emph{All Synthetic} column averages over all 480 instances while the \emph{Real} column averages over 400 real instances (again excluding stochastic lead time). Cell shading per column: \colorbox{cBest}{green} = best, \colorbox{cWorst}{red} = worst. Best method per column in bold.}
\label{tab:gpt5m_by_pattern}
\resizebox{\textwidth}{!}{%
\begin{tabular}{@{}l cccccccccc |c |c@{}}
\toprule
& \rotatebox{70}{Stationary IID} & \rotatebox{70}{Mean $\uparrow$} & \rotatebox{70}{Mean $\downarrow$} & \rotatebox{70}{Trend $\uparrow$} & \rotatebox{70}{Trend $\downarrow$} & \rotatebox{70}{Variance Change} & \rotatebox{70}{Seasonal} & \rotatebox{70}{Multi Changepoint} & \rotatebox{70}{Spike/Dip} & \rotatebox{70}{Autocorr.} & \rotatebox{70}{\emph{All Synthetic}} & \rotatebox{70}{\emph{Real}} \\[-2pt]
& {\scriptsize(p01)} & {\scriptsize(p02)} & {\scriptsize(p03)} & {\scriptsize(p04)} & {\scriptsize(p05)} & {\scriptsize(p06)} & {\scriptsize(p07)} & {\scriptsize(p08)} & {\scriptsize(p09)} & {\scriptsize(p10)} & & \\
\midrule
OR & \cellcolor{cBest} \textbf{.802} & \cellcolor{cWorst!33!cBest} .774 & \cellcolor{cWorst} .557 & \cellcolor{cWorst} .626 & \cellcolor{cWorst} .521 & \cellcolor{cBest} \textbf{.706} & \cellcolor{cBest} \textbf{.676} & \cellcolor{cBest} \textbf{.700} & \cellcolor{cBest} \textbf{.651} & \cellcolor{cBest} \textbf{.756} & \cellcolor{cWorst!45!cBest} .677 & \cellcolor{cWorst} .486 \\
LLM & \cellcolor{cWorst} .703 & \cellcolor{cWorst} .743 & \cellcolor{cBest} \textbf{.627} & \cellcolor{cBest} \textbf{.738} & \cellcolor{cBest} \textbf{.651} & \cellcolor{cWorst} .615 & \cellcolor{cWorst} .616 & \cellcolor{cWorst} .649 & \cellcolor{cWorst} .625 & \cellcolor{cWorst} .667 & \cellcolor{cWorst} .664 & \cellcolor{cBest} \textbf{.526} \\
OR$\to$LLM & \cellcolor{cWorst!75!cBest} .728 & \cellcolor{cWorst!94!cBest} .746 & \cellcolor{cWorst!28!cBest} .608 & \cellcolor{cBest} \textbf{.738} & \cellcolor{cWorst!34!cBest} .607 & \cellcolor{cWorst!67!cBest} .644 & \cellcolor{cWorst!66!cBest} .637 & \cellcolor{cWorst!47!cBest} .676 & \cellcolor{cWorst!52!cBest} .638 & \cellcolor{cWorst!53!cBest} .709 & \cellcolor{cWorst!60!cBest} .673 & \cellcolor{cWorst!5!cBest} .524 \\
LLM$\to$OR & \cellcolor{cWorst!8!cBest} .794 & \cellcolor{cBest} \textbf{.789} & \cellcolor{cWorst!64!cBest} .582 & \cellcolor{cWorst!27!cBest} .709 & \cellcolor{cWorst!70!cBest} .560 & \cellcolor{cWorst!16!cBest} .691 & \cellcolor{cWorst!6!cBest} .672 & \cellcolor{cWorst!6!cBest} .697 & \cellcolor{cWorst!46!cBest} .639 & \cellcolor{cWorst!15!cBest} .743 & \cellcolor{cBest} \textbf{.688} & \cellcolor{cWorst!51!cBest} .506 \\
\bottomrule
\end{tabular}%
}
\end{table}

\begin{table}[H]
\centering
\caption{GPT-5 Mini: implicit critical fractiles, averaged across the instances with a given value of $\rho$.}
\label{tab:gpt5m_implicit_cr}
\begin{tabular}{lcccc}
\toprule
\textbf{Method} & \textbf{$\rho=0.50$} & \textbf{$\rho=0.80$} & \textbf{$\rho=0.95$} & \textbf{Range} \\
\midrule
OR          & 0.470 & 0.527 & 0.597 & 0.127 \\
LLM                         & 0.544 & 0.546 & 0.581 & 0.037 \\
OR$\to$LLM                  & 0.549 & 0.550 & 0.591 & 0.042 \\
LLM$\to$OR                  & 0.508 & 0.530 & 0.600 & 0.092 \\
\bottomrule
\end{tabular}
\end{table}

\section{Detailed Performance Tables} \label{sec:detailed_tables}

Tables~\ref{tab:synthetic_detailed_gemini3flash}--\ref{tab:synthetic_detailed_gpt5mini} provide complete per-pattern, per-lead-time, per-critical-fractile breakdowns for all 720 synthetic instances under each of the three LLMs. Table~\ref{tab:real_detailed} provides the corresponding breakdown for real instances.

\begin{table*}
\centering
\caption{Complete performance breakdown for synthetic instances (Gemini 3 Flash). Each cell shows the average normalized reward over 8 instances. Best method per row is in bold.}
\label{tab:synthetic_detailed_gemini3flash}
\small
\begin{tabular}{@{}llrrrrr@{\hspace{0.8em}}|llrrrrr@{}}
\toprule
\textbf{Pat} & \textbf{L} & \textbf{$\rho$} & \textbf{OR} & \textbf{LLM} & \shortstack{\textbf{OR$\to$}\\\textbf{LLM}} & \shortstack{\textbf{LLM$\to$}\\\textbf{OR}} & \textbf{Pat} & \textbf{L} & \textbf{$\rho$} & \textbf{OR} & \textbf{LLM} & \shortstack{\textbf{OR$\to$}\\\textbf{LLM}} & \shortstack{\textbf{LLM$\to$}\\\textbf{OR}} \\
\midrule
p01 & L=0 & 0.50 & \textbf{0.778} & 0.619 & 0.727 & 0.765 & p06 & L=0 & 0.50 & \textbf{0.684} & 0.552 & 0.623 & 0.664 \\
p01 & L=0 & 0.80 & \textbf{0.905} & 0.883 & 0.900 & 0.896 & p06 & L=0 & 0.80 & \textbf{0.856} & 0.838 & 0.845 & 0.850 \\
p01 & L=0 & 0.95 & \textbf{0.971} & 0.966 & 0.971 & 0.967 & p06 & L=0 & 0.95 & \textbf{0.950} & 0.935 & 0.948 & 0.947 \\
p01 & L=4 & 0.50 & \textbf{0.519} & 0.015 & 0.408 & 0.445 & p06 & L=4 & 0.50 & \textbf{0.294} & 0.000 & 0.166 & 0.145 \\
p01 & L=4 & 0.80 & \textbf{0.764} & 0.593 & 0.751 & 0.740 & p06 & L=4 & 0.80 & \textbf{0.637} & 0.414 & 0.614 & 0.563 \\
p01 & L=4 & 0.95 & \textbf{0.876} & 0.842 & 0.870 & 0.870 & p06 & L=4 & 0.95 & \textbf{0.813} & 0.761 & 0.798 & 0.793 \\
p01 & L=S & 0.50 & 0.069 & 0.028 & \textbf{0.219} & 0.105 & p06 & L=S & 0.50 & 0.077 & 0.001 & \textbf{0.137} & 0.096 \\
p01 & L=S & 0.80 & 0.094 & 0.509 & \textbf{0.523} & 0.236 & p06 & L=S & 0.80 & 0.113 & 0.483 & \textbf{0.499} & 0.254 \\
p01 & L=S & 0.95 & 0.122 & \textbf{0.705} & 0.612 & 0.380 & p06 & L=S & 0.95 & 0.143 & \textbf{0.663} & 0.607 & 0.381 \\
p02 & L=0 & 0.50 & 0.710 & 0.760 & 0.803 & \textbf{0.818} & p07 & L=0 & 0.50 & \textbf{0.699} & 0.584 & 0.641 & 0.683 \\
p02 & L=0 & 0.80 & 0.875 & 0.909 & 0.915 & \textbf{0.915} & p07 & L=0 & 0.80 & \textbf{0.866} & 0.852 & 0.860 & 0.862 \\
p02 & L=0 & 0.95 & 0.955 & 0.961 & \textbf{0.964} & 0.964 & p07 & L=0 & 0.95 & \textbf{0.959} & 0.954 & 0.957 & 0.958 \\
p02 & L=4 & 0.50 & \textbf{0.574} & 0.110 & 0.468 & 0.495 & p07 & L=4 & 0.50 & \textbf{0.144} & 0.000 & 0.079 & 0.108 \\
p02 & L=4 & 0.80 & 0.721 & 0.661 & 0.747 & \textbf{0.754} & p07 & L=4 & 0.80 & \textbf{0.576} & 0.367 & 0.536 & 0.560 \\
p02 & L=4 & 0.95 & 0.809 & 0.819 & 0.838 & \textbf{0.853} & p07 & L=4 & 0.95 & \textbf{0.814} & 0.751 & 0.788 & 0.794 \\
p02 & L=S & 0.50 & 0.059 & 0.077 & \textbf{0.246} & 0.093 & p07 & L=S & 0.50 & 0.057 & 0.013 & \textbf{0.091} & 0.056 \\
p02 & L=S & 0.80 & 0.079 & \textbf{0.511} & 0.482 & 0.237 & p07 & L=S & 0.80 & 0.100 & 0.409 & \textbf{0.484} & 0.216 \\
p02 & L=S & 0.95 & 0.098 & \textbf{0.645} & 0.524 & 0.336 & p07 & L=S & 0.95 & 0.142 & \textbf{0.687} & 0.641 & 0.346 \\
p03 & L=0 & 0.50 & 0.574 & 0.542 & 0.609 & \textbf{0.709} & p08 & L=0 & 0.50 & 0.708 & 0.667 & 0.677 & \textbf{0.740} \\
p03 & L=0 & 0.80 & 0.804 & 0.839 & 0.847 & \textbf{0.868} & p08 & L=0 & 0.80 & 0.865 & 0.871 & 0.873 & \textbf{0.877} \\
p03 & L=0 & 0.95 & 0.937 & 0.953 & 0.950 & \textbf{0.953} & p08 & L=0 & 0.95 & 0.950 & 0.950 & 0.953 & \textbf{0.955} \\
p03 & L=4 & 0.50 & 0.033 & 0.000 & 0.023 & \textbf{0.093} & p08 & L=4 & 0.50 & 0.231 & 0.000 & \textbf{0.241} & 0.193 \\
p03 & L=4 & 0.80 & 0.286 & 0.325 & 0.477 & \textbf{0.539} & p08 & L=4 & 0.80 & 0.630 & 0.497 & \textbf{0.639} & 0.625 \\
p03 & L=4 & 0.95 & 0.707 & 0.744 & 0.762 & \textbf{0.775} & p08 & L=4 & 0.95 & \textbf{0.815} & 0.784 & 0.813 & 0.807 \\
p03 & L=S & 0.50 & \textbf{0.105} & 0.004 & 0.092 & 0.051 & p08 & L=S & 0.50 & 0.065 & 0.031 & \textbf{0.189} & 0.131 \\
p03 & L=S & 0.80 & 0.147 & 0.421 & \textbf{0.500} & 0.284 & p08 & L=S & 0.80 & 0.099 & 0.433 & \textbf{0.484} & 0.246 \\
p03 & L=S & 0.95 & 0.185 & \textbf{0.715} & 0.629 & 0.492 & p08 & L=S & 0.95 & 0.131 & \textbf{0.623} & 0.569 & 0.370 \\
p04 & L=0 & 0.50 & 0.506 & 0.737 & 0.759 & \textbf{0.771} & p09 & L=0 & 0.50 & 0.700 & 0.669 & 0.692 & \textbf{0.752} \\
p04 & L=0 & 0.80 & 0.724 & 0.879 & \textbf{0.898} & 0.885 & p09 & L=0 & 0.80 & 0.846 & 0.873 & 0.868 & \textbf{0.881} \\
p04 & L=0 & 0.95 & 0.896 & 0.950 & \textbf{0.952} & 0.947 & p09 & L=0 & 0.95 & 0.944 & 0.950 & \textbf{0.954} & 0.952 \\
p04 & L=4 & 0.50 & 0.408 & 0.390 & 0.431 & \textbf{0.569} & p09 & L=4 & 0.50 & \textbf{0.105} & 0.000 & 0.026 & 0.073 \\
p04 & L=4 & 0.80 & 0.556 & 0.704 & 0.719 & \textbf{0.742} & p09 & L=4 & 0.80 & 0.542 & 0.341 & 0.561 & \textbf{0.572} \\
p04 & L=4 & 0.95 & 0.670 & \textbf{0.835} & 0.806 & 0.830 & p09 & L=4 & 0.95 & 0.769 & 0.737 & 0.766 & \textbf{0.779} \\
p04 & L=S & 0.50 & 0.061 & 0.189 & \textbf{0.284} & 0.146 & p09 & L=S & 0.50 & 0.065 & 0.044 & \textbf{0.168} & 0.099 \\
p04 & L=S & 0.80 & 0.083 & \textbf{0.487} & 0.445 & 0.248 & p09 & L=S & 0.80 & 0.104 & 0.441 & \textbf{0.450} & 0.256 \\
p04 & L=S & 0.95 & 0.106 & \textbf{0.595} & 0.482 & 0.314 & p09 & L=S & 0.95 & 0.133 & 0.567 & \textbf{0.571} & 0.308 \\
p05 & L=0 & 0.50 & 0.514 & 0.689 & 0.604 & \textbf{0.766} & p10 & L=0 & 0.50 & \textbf{0.770} & 0.674 & 0.714 & 0.755 \\
p05 & L=0 & 0.80 & 0.787 & 0.888 & 0.850 & \textbf{0.894} & p10 & L=0 & 0.80 & \textbf{0.900} & 0.889 & 0.895 & 0.890 \\
p05 & L=0 & 0.95 & 0.937 & 0.962 & 0.958 & \textbf{0.967} & p10 & L=0 & 0.95 & 0.968 & 0.964 & \textbf{0.968} & 0.959 \\
p05 & L=4 & 0.50 & 0.000 & 0.000 & 0.000 & \textbf{0.046} & p10 & L=4 & 0.50 & \textbf{0.369} & 0.010 & 0.192 & 0.242 \\
p05 & L=4 & 0.80 & 0.231 & 0.310 & 0.449 & \textbf{0.513} & p10 & L=4 & 0.80 & \textbf{0.688} & 0.487 & 0.656 & 0.645 \\
p05 & L=4 & 0.95 & 0.655 & 0.730 & 0.739 & \textbf{0.754} & p10 & L=4 & 0.95 & \textbf{0.840} & 0.796 & 0.821 & 0.813 \\
p05 & L=S & 0.50 & \textbf{0.100} & 0.000 & 0.043 & 0.055 & p10 & L=S & 0.50 & 0.077 & 0.057 & \textbf{0.217} & 0.099 \\
p05 & L=S & 0.80 & 0.151 & 0.418 & \textbf{0.459} & 0.337 & p10 & L=S & 0.80 & 0.108 & \textbf{0.491} & 0.485 & 0.282 \\
p05 & L=S & 0.95 & 0.197 & \textbf{0.671} & 0.591 & 0.417 & p10 & L=S & 0.95 & 0.132 & \textbf{0.676} & 0.598 & 0.431 \\
\bottomrule
\end{tabular}
\vspace{0.5em}

\raggedright
\textbf{Pattern abbreviations:} p01=Stationary IID, p02=Mean Increase, p03=Mean Decrease, p04=Increasing Trend, p05=Decreasing Trend, p06=Variance Change, p07=Seasonal, p08=Multi Changepoint, p09=Temp Spike/Dip, p10=Autocorrelated
\end{table*}

\begin{table*}
\centering
\caption{Complete performance breakdown for synthetic instances (Grok 4.1 Fast). Each cell shows the average normalized reward over 8 instances. Best method per row is in bold.}
\label{tab:synthetic_detailed_grok41}
\small
\begin{tabular}{@{}llrrrrr@{\hspace{0.8em}}|llrrrrr@{}}
\toprule
\textbf{Pat} & \textbf{L} & \textbf{$\rho$} & \textbf{OR} & \textbf{LLM} & \shortstack{\textbf{OR$\to$}\\\textbf{LLM}} & \shortstack{\textbf{LLM$\to$}\\\textbf{OR}} & \textbf{Pat} & \textbf{L} & \textbf{$\rho$} & \textbf{OR} & \textbf{LLM} & \shortstack{\textbf{OR$\to$}\\\textbf{LLM}} & \shortstack{\textbf{LLM$\to$}\\\textbf{OR}} \\
\midrule
p01 & L=0 & 0.50 & \textbf{0.778} & 0.716 & 0.760 & 0.766 & p06 & L=0 & 0.50 & \textbf{0.684} & 0.594 & 0.654 & 0.670 \\
p01 & L=0 & 0.80 & \textbf{0.905} & 0.896 & 0.901 & 0.898 & p06 & L=0 & 0.80 & \textbf{0.856} & 0.844 & 0.852 & 0.848 \\
p01 & L=0 & 0.95 & \textbf{0.971} & 0.957 & 0.969 & 0.960 & p06 & L=0 & 0.95 & 0.950 & 0.934 & \textbf{0.952} & 0.944 \\
p01 & L=4 & 0.50 & \textbf{0.519} & 0.022 & 0.472 & 0.447 & p06 & L=4 & 0.50 & \textbf{0.294} & 0.004 & 0.180 & 0.180 \\
p01 & L=4 & 0.80 & \textbf{0.764} & 0.298 & 0.762 & 0.736 & p06 & L=4 & 0.80 & \textbf{0.637} & 0.151 & 0.623 & 0.562 \\
p01 & L=4 & 0.95 & \textbf{0.876} & 0.752 & 0.876 & 0.865 & p06 & L=4 & 0.95 & \textbf{0.813} & 0.543 & 0.808 & 0.792 \\
p01 & L=S & 0.50 & 0.069 & \textbf{0.107} & 0.099 & 0.000 & p06 & L=S & 0.50 & 0.077 & 0.070 & \textbf{0.110} & 0.000 \\
p01 & L=S & 0.80 & 0.094 & \textbf{0.512} & 0.444 & 0.132 & p06 & L=S & 0.80 & 0.113 & \textbf{0.498} & 0.388 & 0.117 \\
p01 & L=S & 0.95 & 0.122 & \textbf{0.707} & 0.637 & 0.548 & p06 & L=S & 0.95 & 0.143 & \textbf{0.703} & 0.563 & 0.486 \\
p02 & L=0 & 0.50 & 0.710 & 0.790 & 0.795 & \textbf{0.820} & p07 & L=0 & 0.50 & \textbf{0.699} & 0.647 & 0.678 & 0.684 \\
p02 & L=0 & 0.80 & 0.875 & 0.904 & 0.908 & \textbf{0.914} & p07 & L=0 & 0.80 & \textbf{0.866} & 0.857 & 0.860 & 0.860 \\
p02 & L=0 & 0.95 & 0.955 & 0.951 & 0.960 & \textbf{0.963} & p07 & L=0 & 0.95 & 0.959 & 0.943 & \textbf{0.959} & 0.951 \\
p02 & L=4 & 0.50 & \textbf{0.574} & 0.207 & 0.465 & 0.521 & p07 & L=4 & 0.50 & \textbf{0.144} & 0.000 & 0.119 & 0.094 \\
p02 & L=4 & 0.80 & 0.721 & 0.462 & 0.750 & \textbf{0.760} & p07 & L=4 & 0.80 & \textbf{0.576} & 0.152 & 0.560 & 0.549 \\
p02 & L=4 & 0.95 & 0.809 & 0.802 & 0.848 & \textbf{0.856} & p07 & L=4 & 0.95 & \textbf{0.814} & 0.656 & 0.812 & 0.788 \\
p02 & L=S & 0.50 & 0.059 & \textbf{0.178} & 0.169 & 0.000 & p07 & L=S & 0.50 & \textbf{0.057} & 0.012 & 0.051 & 0.023 \\
p02 & L=S & 0.80 & 0.079 & \textbf{0.499} & 0.457 & 0.186 & p07 & L=S & 0.80 & 0.100 & 0.399 & \textbf{0.412} & 0.210 \\
p02 & L=S & 0.95 & 0.098 & \textbf{0.665} & 0.562 & 0.430 & p07 & L=S & 0.95 & 0.142 & \textbf{0.683} & 0.571 & 0.505 \\
p03 & L=0 & 0.50 & 0.574 & 0.610 & 0.641 & \textbf{0.700} & p08 & L=0 & 0.50 & 0.708 & 0.680 & 0.716 & \textbf{0.744} \\
p03 & L=0 & 0.80 & 0.804 & 0.859 & 0.851 & \textbf{0.873} & p08 & L=0 & 0.80 & 0.865 & 0.860 & 0.869 & \textbf{0.878} \\
p03 & L=0 & 0.95 & 0.937 & \textbf{0.952} & 0.949 & 0.948 & p08 & L=0 & 0.95 & 0.950 & 0.939 & \textbf{0.952} & 0.952 \\
p03 & L=4 & 0.50 & 0.033 & 0.000 & 0.032 & \textbf{0.079} & p08 & L=4 & 0.50 & \textbf{0.231} & 0.008 & 0.165 & 0.172 \\
p03 & L=4 & 0.80 & 0.286 & 0.076 & 0.400 & \textbf{0.539} & p08 & L=4 & 0.80 & \textbf{0.630} & 0.261 & 0.618 & 0.618 \\
p03 & L=4 & 0.95 & 0.707 & 0.459 & 0.739 & \textbf{0.778} & p08 & L=4 & 0.95 & \textbf{0.815} & 0.689 & 0.815 & 0.812 \\
p03 & L=S & 0.50 & \textbf{0.105} & 0.009 & 0.095 & 0.010 & p08 & L=S & 0.50 & \textbf{0.065} & 0.032 & 0.059 & 0.010 \\
p03 & L=S & 0.80 & 0.147 & 0.388 & \textbf{0.403} & 0.313 & p08 & L=S & 0.80 & 0.099 & \textbf{0.407} & 0.373 & 0.275 \\
p03 & L=S & 0.95 & 0.185 & \textbf{0.723} & 0.543 & 0.555 & p08 & L=S & 0.95 & 0.131 & \textbf{0.671} & 0.550 & 0.525 \\
p04 & L=0 & 0.50 & 0.506 & 0.753 & \textbf{0.765} & 0.752 & p09 & L=0 & 0.50 & 0.700 & 0.698 & 0.706 & \textbf{0.749} \\
p04 & L=0 & 0.80 & 0.724 & 0.889 & \textbf{0.892} & 0.877 & p09 & L=0 & 0.80 & 0.846 & 0.869 & 0.869 & \textbf{0.879} \\
p04 & L=0 & 0.95 & 0.896 & \textbf{0.951} & 0.949 & 0.945 & p09 & L=0 & 0.95 & 0.944 & 0.942 & \textbf{0.954} & 0.949 \\
p04 & L=4 & 0.50 & 0.408 & 0.202 & \textbf{0.558} & 0.549 & p09 & L=4 & 0.50 & \textbf{0.105} & 0.000 & 0.024 & 0.049 \\
p04 & L=4 & 0.80 & 0.556 & 0.566 & \textbf{0.744} & 0.689 & p09 & L=4 & 0.80 & 0.542 & 0.150 & 0.558 & \textbf{0.574} \\
p04 & L=4 & 0.95 & 0.670 & 0.819 & \textbf{0.825} & 0.807 & p09 & L=4 & 0.95 & 0.769 & 0.557 & 0.773 & \textbf{0.773} \\
p04 & L=S & 0.50 & 0.061 & 0.158 & \textbf{0.242} & 0.000 & p09 & L=S & 0.50 & \textbf{0.065} & 0.025 & 0.053 & 0.006 \\
p04 & L=S & 0.80 & 0.083 & \textbf{0.452} & 0.430 & 0.160 & p09 & L=S & 0.80 & 0.104 & \textbf{0.392} & 0.381 & 0.232 \\
p04 & L=S & 0.95 & 0.106 & \textbf{0.651} & 0.516 & 0.465 & p09 & L=S & 0.95 & 0.133 & \textbf{0.627} & 0.529 & 0.498 \\
p05 & L=0 & 0.50 & 0.514 & 0.714 & 0.673 & \textbf{0.753} & p10 & L=0 & 0.50 & \textbf{0.770} & 0.717 & 0.754 & 0.764 \\
p05 & L=0 & 0.80 & 0.787 & \textbf{0.897} & 0.868 & 0.894 & p10 & L=0 & 0.80 & \textbf{0.900} & 0.890 & 0.898 & 0.891 \\
p05 & L=0 & 0.95 & 0.937 & 0.957 & 0.960 & \textbf{0.966} & p10 & L=0 & 0.95 & \textbf{0.968} & 0.952 & 0.965 & 0.960 \\
p05 & L=4 & 0.50 & 0.000 & 0.000 & 0.000 & \textbf{0.063} & p10 & L=4 & 0.50 & \textbf{0.369} & 0.067 & 0.276 & 0.238 \\
p05 & L=4 & 0.80 & 0.231 & 0.053 & 0.395 & \textbf{0.541} & p10 & L=4 & 0.80 & \textbf{0.688} & 0.362 & 0.665 & 0.645 \\
p05 & L=4 & 0.95 & 0.655 & 0.511 & 0.730 & \textbf{0.770} & p10 & L=4 & 0.95 & \textbf{0.840} & 0.670 & 0.831 & 0.815 \\
p05 & L=S & 0.50 & \textbf{0.100} & 0.012 & 0.058 & 0.000 & p10 & L=S & 0.50 & 0.077 & 0.079 & \textbf{0.153} & 0.026 \\
p05 & L=S & 0.80 & 0.151 & \textbf{0.416} & 0.359 & 0.191 & p10 & L=S & 0.80 & 0.108 & \textbf{0.483} & 0.438 & 0.195 \\
p05 & L=S & 0.95 & 0.197 & \textbf{0.651} & 0.556 & 0.510 & p10 & L=S & 0.95 & 0.132 & \textbf{0.666} & 0.524 & 0.486 \\
\bottomrule
\end{tabular}
\vspace{0.5em}

\raggedright
\textbf{Pattern abbreviations:} p01=Stationary IID, p02=Mean Increase, p03=Mean Decrease, p04=Increasing Trend, p05=Decreasing Trend, p06=Variance Change, p07=Seasonal, p08=Multi Changepoint, p09=Temp Spike/Dip, p10=Autocorrelated
\end{table*}

\begin{table*}
\centering
\caption{Complete performance breakdown for synthetic instances (GPT-5 Mini). Each cell shows the average normalized reward over 8 instances. Best method per row is in bold.}
\label{tab:synthetic_detailed_gpt5mini}
\small
\begin{tabular}{@{}llrrrrr@{\hspace{0.8em}}|llrrrrr@{}}
\toprule
\textbf{Pat} & \textbf{L} & \textbf{$\rho$} & \textbf{OR} & \textbf{LLM} & \shortstack{\textbf{OR$\to$}\\\textbf{LLM}} & \shortstack{\textbf{LLM$\to$}\\\textbf{OR}} & \textbf{Pat} & \textbf{L} & \textbf{$\rho$} & \textbf{OR} & \textbf{LLM} & \shortstack{\textbf{OR$\to$}\\\textbf{LLM}} & \shortstack{\textbf{LLM$\to$}\\\textbf{OR}} \\
\midrule
p01 & L=0 & 0.50 & \textbf{0.778} & 0.751 & 0.766 & 0.766 & p06 & L=0 & 0.50 & \textbf{0.684} & 0.664 & 0.676 & 0.679 \\
p01 & L=0 & 0.80 & \textbf{0.905} & 0.866 & 0.887 & 0.902 & p06 & L=0 & 0.80 & \textbf{0.856} & 0.816 & 0.840 & 0.856 \\
p01 & L=0 & 0.95 & \textbf{0.971} & 0.923 & 0.956 & 0.970 & p06 & L=0 & 0.95 & \textbf{0.950} & 0.883 & 0.921 & 0.946 \\
p01 & L=4 & 0.50 & \textbf{0.519} & 0.187 & 0.229 & 0.493 & p06 & L=4 & 0.50 & \textbf{0.294} & 0.053 & 0.090 & 0.238 \\
p01 & L=4 & 0.80 & \textbf{0.764} & 0.682 & 0.702 & 0.761 & p06 & L=4 & 0.80 & \textbf{0.637} & 0.528 & 0.588 & 0.614 \\
p01 & L=4 & 0.95 & \textbf{0.876} & 0.809 & 0.831 & 0.875 & p06 & L=4 & 0.95 & 0.813 & 0.744 & 0.751 & \textbf{0.815} \\
p01 & L=S & 0.50 & 0.069 & 0.031 & 0.066 & \textbf{0.093} & p06 & L=S & 0.50 & 0.077 & 0.060 & 0.085 & \textbf{0.105} \\
p01 & L=S & 0.80 & 0.094 & \textbf{0.189} & 0.163 & 0.145 & p06 & L=S & 0.80 & 0.113 & \textbf{0.232} & 0.187 & 0.172 \\
p01 & L=S & 0.95 & 0.122 & \textbf{0.257} & 0.223 & 0.179 & p06 & L=S & 0.95 & 0.143 & \textbf{0.338} & 0.256 & 0.213 \\
p02 & L=0 & 0.50 & 0.710 & 0.816 & \textbf{0.819} & 0.805 & p07 & L=0 & 0.50 & 0.699 & 0.693 & 0.696 & \textbf{0.704} \\
p02 & L=0 & 0.80 & 0.875 & 0.887 & 0.899 & \textbf{0.909} & p07 & L=0 & 0.80 & \textbf{0.866} & 0.838 & 0.845 & 0.861 \\
p02 & L=0 & 0.95 & 0.955 & 0.919 & 0.952 & \textbf{0.961} & p07 & L=0 & 0.95 & \textbf{0.959} & 0.902 & 0.946 & 0.956 \\
p02 & L=4 & 0.50 & \textbf{0.574} & 0.323 & 0.336 & 0.490 & p07 & L=4 & 0.50 & \textbf{0.144} & 0.018 & 0.082 & 0.133 \\
p02 & L=4 & 0.80 & 0.721 & 0.714 & 0.689 & \textbf{0.739} & p07 & L=4 & 0.80 & \textbf{0.576} & 0.501 & 0.507 & 0.572 \\
p02 & L=4 & 0.95 & 0.809 & 0.802 & 0.782 & \textbf{0.832} & p07 & L=4 & 0.95 & \textbf{0.814} & 0.742 & 0.742 & 0.808 \\
p02 & L=S & 0.50 & 0.059 & \textbf{0.085} & 0.070 & 0.084 & p07 & L=S & 0.50 & 0.057 & 0.030 & 0.019 & \textbf{0.091} \\
p02 & L=S & 0.80 & 0.079 & 0.132 & 0.125 & \textbf{0.133} & p07 & L=S & 0.80 & 0.100 & \textbf{0.246} & 0.171 & 0.155 \\
p02 & L=S & 0.95 & 0.098 & 0.157 & \textbf{0.157} & 0.141 & p07 & L=S & 0.95 & 0.142 & \textbf{0.360} & 0.211 & 0.220 \\
p03 & L=0 & 0.50 & 0.574 & \textbf{0.715} & 0.698 & 0.663 & p08 & L=0 & 0.50 & 0.708 & 0.739 & \textbf{0.744} & 0.721 \\
p03 & L=0 & 0.80 & 0.804 & 0.851 & \textbf{0.865} & 0.845 & p08 & L=0 & 0.80 & 0.865 & 0.858 & \textbf{0.870} & 0.870 \\
p03 & L=0 & 0.95 & 0.937 & 0.911 & 0.945 & \textbf{0.951} & p08 & L=0 & 0.95 & 0.950 & 0.906 & 0.941 & \textbf{0.952} \\
p03 & L=4 & 0.50 & 0.033 & 0.030 & \textbf{0.047} & 0.032 & p08 & L=4 & 0.50 & \textbf{0.231} & 0.023 & 0.105 & 0.189 \\
p03 & L=4 & 0.80 & 0.286 & \textbf{0.509} & 0.345 & 0.290 & p08 & L=4 & 0.80 & 0.630 & 0.596 & 0.610 & \textbf{0.633} \\
p03 & L=4 & 0.95 & 0.707 & \textbf{0.748} & 0.746 & 0.712 & p08 & L=4 & 0.95 & \textbf{0.815} & 0.773 & 0.784 & 0.814 \\
p03 & L=S & 0.50 & 0.105 & 0.019 & 0.057 & \textbf{0.122} & p08 & L=S & 0.50 & 0.065 & 0.057 & 0.095 & \textbf{0.105} \\
p03 & L=S & 0.80 & 0.147 & 0.252 & \textbf{0.257} & 0.218 & p08 & L=S & 0.80 & 0.099 & \textbf{0.199} & 0.171 & 0.161 \\
p03 & L=S & 0.95 & 0.185 & \textbf{0.395} & 0.330 & 0.260 & p08 & L=S & 0.95 & 0.131 & \textbf{0.307} & 0.179 & 0.216 \\
p04 & L=0 & 0.50 & 0.506 & \textbf{0.771} & 0.761 & 0.644 & p09 & L=0 & 0.50 & 0.700 & 0.734 & \textbf{0.734} & 0.722 \\
p04 & L=0 & 0.80 & 0.724 & \textbf{0.854} & 0.853 & 0.838 & p09 & L=0 & 0.80 & 0.846 & 0.849 & \textbf{0.869} & 0.865 \\
p04 & L=0 & 0.95 & 0.896 & 0.892 & 0.912 & \textbf{0.938} & p09 & L=0 & 0.95 & 0.944 & 0.904 & 0.938 & \textbf{0.945} \\
p04 & L=4 & 0.50 & 0.408 & 0.479 & \textbf{0.480} & 0.460 & p09 & L=4 & 0.50 & \textbf{0.105} & 0.000 & 0.019 & 0.025 \\
p04 & L=4 & 0.80 & 0.556 & 0.661 & \textbf{0.668} & 0.639 & p09 & L=4 & 0.80 & \textbf{0.542} & 0.524 & 0.524 & 0.505 \\
p04 & L=4 & 0.95 & 0.670 & \textbf{0.773} & 0.756 & 0.733 & p09 & L=4 & 0.95 & 0.769 & 0.741 & 0.741 & \textbf{0.773} \\
p04 & L=S & 0.50 & 0.061 & \textbf{0.114} & 0.086 & 0.108 & p09 & L=S & 0.50 & 0.065 & 0.042 & 0.075 & \textbf{0.083} \\
p04 & L=S & 0.80 & 0.083 & 0.168 & \textbf{0.173} & 0.165 & p09 & L=S & 0.80 & 0.104 & \textbf{0.204} & 0.197 & 0.149 \\
p04 & L=S & 0.95 & 0.106 & \textbf{0.238} & 0.184 & 0.178 & p09 & L=S & 0.95 & 0.133 & 0.308 & \textbf{0.322} & 0.201 \\
p05 & L=0 & 0.50 & 0.514 & \textbf{0.786} & 0.758 & 0.671 & p10 & L=0 & 0.50 & 0.770 & 0.762 & \textbf{0.771} & 0.762 \\
p05 & L=0 & 0.80 & 0.787 & 0.892 & \textbf{0.900} & 0.847 & p10 & L=0 & 0.80 & \textbf{0.900} & 0.867 & 0.890 & 0.896 \\
p05 & L=0 & 0.95 & 0.937 & 0.943 & \textbf{0.962} & 0.943 & p10 & L=0 & 0.95 & \textbf{0.968} & 0.920 & 0.958 & 0.966 \\
p05 & L=4 & 0.50 & 0.000 & \textbf{0.012} & 0.007 & 0.000 & p10 & L=4 & 0.50 & \textbf{0.369} & 0.075 & 0.233 & 0.322 \\
p05 & L=4 & 0.80 & 0.231 & \textbf{0.514} & 0.309 & 0.242 & p10 & L=4 & 0.80 & \textbf{0.688} & 0.604 & 0.612 & 0.678 \\
p05 & L=4 & 0.95 & 0.655 & \textbf{0.760} & 0.707 & 0.657 & p10 & L=4 & 0.95 & \textbf{0.840} & 0.776 & 0.790 & 0.832 \\
p05 & L=S & 0.50 & 0.100 & 0.042 & 0.077 & \textbf{0.127} & p10 & L=S & 0.50 & 0.077 & 0.052 & 0.070 & \textbf{0.096} \\
p05 & L=S & 0.80 & 0.151 & \textbf{0.258} & 0.192 & 0.219 & p10 & L=S & 0.80 & 0.108 & \textbf{0.223} & 0.164 & 0.149 \\
p05 & L=S & 0.95 & 0.197 & \textbf{0.342} & 0.341 & 0.300 & p10 & L=S & 0.95 & 0.132 & \textbf{0.290} & 0.223 & 0.191 \\
\bottomrule
\end{tabular}
\vspace{0.5em}

\raggedright
\textbf{Pattern abbreviations:} p01=Stationary IID, p02=Mean Increase, p03=Mean Decrease, p04=Increasing Trend, p05=Decreasing Trend, p06=Variance Change, p07=Seasonal, p08=Multi Changepoint, p09=Temp Spike/Dip, p10=Autocorrelated
\end{table*}

\begin{table}
\centering
\caption{Performance breakdown for real instances (H\&M fashion products) across all three LLMs. Each group of 3 rows (same lead time) averages over all 200 articles. Best method per row is in bold.}
\label{tab:real_detailed}
\begin{tabular}{llrrrr}
\toprule
\textbf{Lead Time} & \textbf{$\rho$} & \textbf{OR} & \textbf{LLM} & \shortstack{\textbf{OR$\to$}\\\textbf{LLM}} & \shortstack{\textbf{LLM$\to$}\\\textbf{OR}} \\
\midrule
\multicolumn{6}{l}{\textit{Gemini 3 Flash}} \\
L=0          & 0.50 & 0.433 & 0.546 & 0.554 & \textbf{0.642} \\
L=0          & 0.80 & 0.698 & \textbf{0.779} & 0.770 & 0.769 \\
L=0          & 0.95 & 0.882 & 0.889 & \textbf{0.909} & 0.901 \\
L=4          & 0.50 & \textbf{0.046} & 0.000 & 0.016 & 0.020 \\
L=4          & 0.80 & 0.255 & 0.117 & \textbf{0.293} & 0.275 \\
L=4          & 0.95 & 0.647 & 0.629 & \textbf{0.667} & 0.666 \\
Stochastic   & 0.50 & \textbf{0.152} & 0.001 & 0.070 & 0.055 \\
Stochastic   & 0.80 & 0.222 & 0.279 & \textbf{0.385} & 0.236 \\
Stochastic   & 0.95 & 0.250 & \textbf{0.623} & 0.578 & 0.404 \\
\midrule
\multicolumn{6}{l}{\textit{Grok 4.1 Fast}} \\
L=0          & 0.50 & 0.433 & 0.501 & 0.452 & \textbf{0.628} \\
L=0          & 0.80 & 0.698 & 0.727 & 0.722 & \textbf{0.777} \\
L=0          & 0.95 & 0.882 & 0.892 & 0.902 & \textbf{0.903} \\
L=4          & 0.50 & \textbf{0.046} & 0.000 & 0.012 & 0.024 \\
L=4          & 0.80 & 0.255 & 0.065 & 0.246 & \textbf{0.286} \\
L=4          & 0.95 & 0.647 & 0.417 & 0.665 & \textbf{0.668} \\
Stochastic   & 0.50 & \textbf{0.152} & 0.021 & 0.058 & 0.006 \\
Stochastic   & 0.80 & 0.222 & 0.248 & \textbf{0.312} & 0.136 \\
Stochastic   & 0.95 & 0.250 & \textbf{0.622} & 0.517 & 0.431 \\
\midrule
\multicolumn{6}{l}{\textit{GPT-5 Mini}} \\
L=0          & 0.50 & 0.433 & \textbf{0.659} & 0.631 & 0.556 \\
L=0          & 0.80 & 0.698 & 0.778 & \textbf{0.779} & 0.744 \\
L=0          & 0.95 & 0.882 & 0.857 & 0.884 & \textbf{0.895} \\
L=4          & 0.50 & \textbf{0.046} & 0.004 & 0.012 & 0.014 \\
L=4          & 0.80 & 0.255 & \textbf{0.268} & 0.266 & 0.225 \\
L=4          & 0.95 & \textbf{0.647} & 0.644 & 0.624 & 0.643 \\
Stochastic   & 0.50 & \textbf{0.152} & 0.039 & 0.061 & 0.118 \\
Stochastic   & 0.80 & \textbf{0.222} & 0.166 & 0.199 & 0.167 \\
Stochastic   & 0.95 & 0.250 & \textbf{0.275} & 0.274 & 0.214 \\
\bottomrule
\end{tabular}
\end{table}

\section{LLM and Human Reasoning Anecdotes} \label{sec:llm_reasoning_anecdotes}

This appendix presents representative examples of LLM reasoning from the benchmark experiments, illustrating key capabilities (and one downside) that distinguish LLM-based decision-making from pure OR methods. \Cref{tab:human_reasoning_examples} further shows examples of human guidance messages from the classroom experiment, illustrating how humans add value through the same contextual reasoning channels.

\begin{example}[LLM detects lost orders] \label{ex:llm_lost_orders}
In instance 108775044 (an H\&M strap top) under stochastic lead times, at period 10, the OR algorithm recommends ordering 0 units because it counts 936 units as pipeline inventory. The Gemini-based OR$\to$LLM hybrid detects this error and overrides:

\textit{``Critically, several shipments (P1, P3, P4) are confirmed lost, as they never arrived despite being 4+ periods old\ldots\ The OR recommends 0 because it counts the full 936 units as pipeline inventory. I must ignore the OR recommendation because the pipeline is corrupted by lost orders. With an effective pipeline of 390 units and demand averaging $\sim$230 per period\ldots\ Current effective inventory = 390. Shortfall $\approx$ 300 units.''}

The LLM recalculates effective inventory (936 $\to$ 390 units) and overrides OR's recommendation of 0, ordering 210 units instead.
\end{example}

\begin{example}[LLM overfits] \label{ex:llm_overfits}
While LLMs excel at detecting genuine demand shifts, they can also overfit to natural variance in stationary settings. Consider a stationary IID instance drawn from $N(100, 25^2)$ (actual mean: 102.4, std: 22.9). In Period~9, demand was 145---a high draw but well within two standard deviations. Gemini interpreted this as a regime shift, reasoning at Period~10: \textit{``Demand analysis shows an upward trend over the last three periods: P7~(122), P8~(121), and P9~(145), indicating the mean has shifted from the previous $\sim$100 range to approximately 130+.''} It created a carry-over insight: \textit{``Demand regime shift at Period~9: Demand increased from an average of $\sim$100 to 145 (+45\%), potentially indicating a new higher demand baseline,''} and forecast Period~10 demand at 145--150.

However, subsequent periods (P10:~115, P11:~89, P12:~102, P13:~73, P14:~63) all fell near the true mean of~100, confirming no regime shift occurred. Despite this, Gemini continued pattern-seeking, updating its carry-over insights at Period~12 (\textit{``Demand regime stabilized further: average demand from Period~9--12 is $\sim$113, adjusting previous Period~10 insight down from 145''}), Period~14 (\textit{``Recent demand (P12--P14) shows a notable downward trend (avg~79) compared to the previous regime (P9--P11 avg~116), suggesting a cooling of the earlier demand spike''}), and Period~24 (\textit{``Demand stability extended: Period~20--24 demand mean is 105.0 with a range of 102--107, confirming a high-stability regime following earlier volatility''}). This illustrates how LLMs can impose narratives on random noise, a weakness that helps explain why pure OR outperforms LLM on stationary IID instances (OR:~0.78 vs.\ LLM:~0.62 at $\rho=0.50$, $L=0$; see Table~\ref{tab:synthetic_detailed_gemini3flash}).
\end{example}

\begin{example}[LLM detects synthetic demand shift] \label{ex:llm_demand_shift}
Consider a mean increase instance (pattern p02) where demand shifts from 100 to 200 units at period 16. At period 17, the LLM reasons:

\textit{``The recent demand in Period 16 spiked massively to 273 units, far exceeding the previous range of 46--137. This suggests a potential high-demand outlier or new seasonal regime\ldots\ Before the P16 spike, the demand average was approximately 100.''}

By period 18, the LLM confirms the shift:

\textit{``Sustained high-demand regime confirmed: Period 16 (273) and Period 17 (170) indicate a significant upward shift from the previous average of $\sim$100. Calculating safety stock based on a mean of $\sim$220 for the near term.''}

Over successive periods, the LLM detects the changepoint, adapts its forecast (100 $\to$ 220), and monitors for confirmation.
\end{example}

\begin{example}[LLM uses world knowledge] \label{ex:llm_world_knowledge}
Consider instance 599580017, the ``Timeless Midrise Brief'' (women's swimwear). In early spring, the LLM reasons:

\textit{``Based on the date 2019-03-11, we are entering the Spring season (Northern Hemisphere), which is a high-demand period for beachwear.''}

By November, the same LLM adjusts:

\textit{``We are currently in mid-November (Period 40), which is deep into the off-season for swimwear.''}

This calendar-aware reasoning allows the LLM to incorporate seasonal knowledge that the OR algorithm cannot access.
\end{example}

\begin{table}[H]
    \centering
    \small
    \begin{tabular}{p{0.30\textwidth}p{0.30\textwidth}p{0.30\textwidth}}
    \toprule
    \textbf{Instance 1 (Swimwear)} & \textbf{Instance 2 (Blazer)} & \textbf{Instance 3 (Trousers)} \\
    \textbf{Demand trend detection} & \textbf{Seasonal reasoning} & \textbf{Detecting anomalies (lost orders)} \\
    \midrule

    \textit{``it seems that the demand is going down, so I would rather keep the conservative strategy''}

    \vspace{3pt}

    \textit{``demand decreased sharply... go one week without ordering''}
    \vspace{3pt}

    \textit{``As demand has shifted to a low level, reduce the inventory target''}
    \vspace{3pt}

    \textit{``seems like the demand is lower than before, so make conservative choices''}
    \vspace{3pt}

    &
    \textit{``Blazers are a very seasonal product, make sure to rather over estimate during fall/winter. Estimate very conservatively during spring/summer.''}

    \vspace{3pt}

    \textit{``given that a blazer is for spring/autumn weather... Be conservative for summer, but don't understock in autumn.''}

    \vspace{3pt}

    \textit{``less blazers in summer... peak before fall/autumn''}
    &
    \textit{``some orders may never arrive, make more generous orders''}

    \vspace{3pt}

    \textit{``in-transit inventory is inaccurate... ignore the in-transit inventory''}

    \vspace{3pt}

    \textit{``make way more generous orders... hedge against orders that may not arrive.''}

    \vspace{3pt}

    \textit{``be very aggressive since very few orders are arriving''}
    \\
    \bottomrule
    \end{tabular}
    \caption{Examples of human guidance messages for the three instances illustrating three types of reasoning.
    } \label{tab:human_reasoning_examples}
\end{table}

\clearpage
\section{Details for Human Experiments and Inventory Game} \label{sec:human_instances}

\begin{figure}
    \centering
    \includegraphics[width=0.85\textwidth]{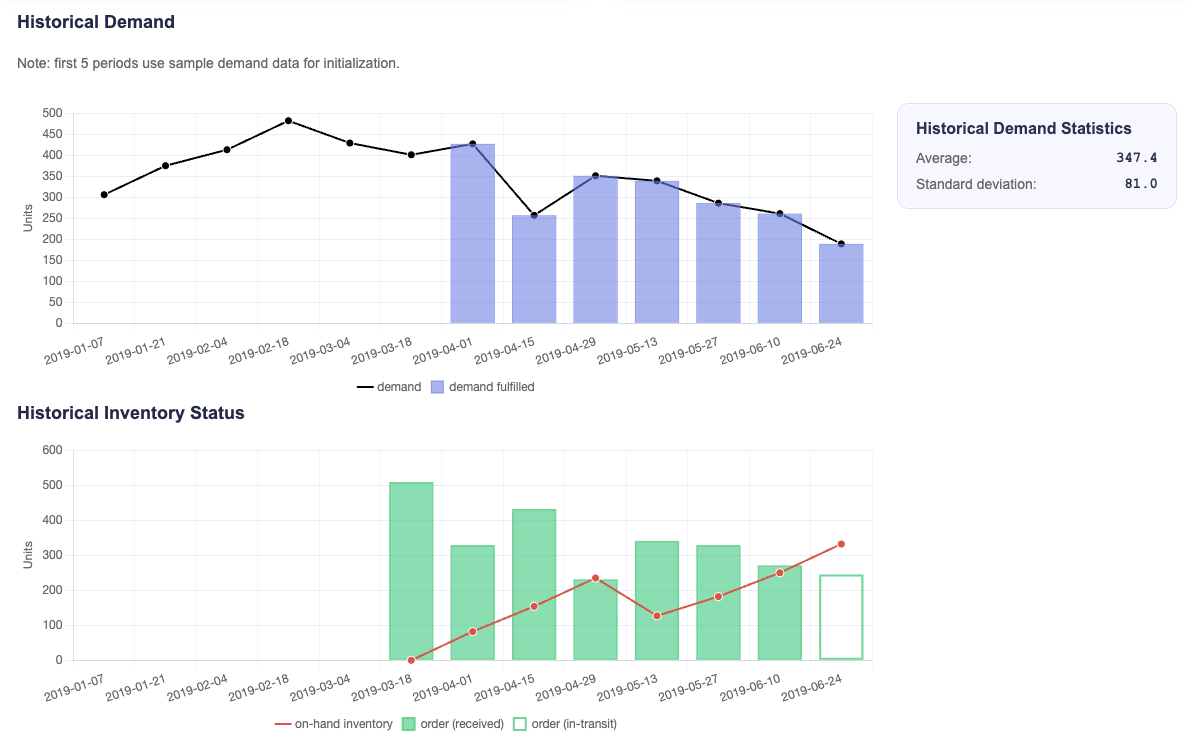}
    \caption{Analytics panel of the inventory game: historical demand chart with summary statistics (left) and historical inventory status (right). Together with the decision panel shown in Figure~\ref{fig:game_interface}, these displays give participants all information needed for each ordering decision.}
    \label{fig:interface_bottom}
\end{figure}

\cref{fig:human_instances_description} shows the product image and description for each of the three instances that were displayed to the users, and \cref{fig:human_instances_demands} plots their realized demand patterns over time.
The instances also differ in their lead-time configurations: Instance~1 has a lead time of~0, Instance~2 has a lead time of~1, and Instance~3 has a stochastic lead time that equals~1 with probability 75\% and $\infty$ otherwise.

\begin{figure}
    \centering
    \begin{subfigure}{0.68\textwidth}
        \includegraphics[width=\linewidth]{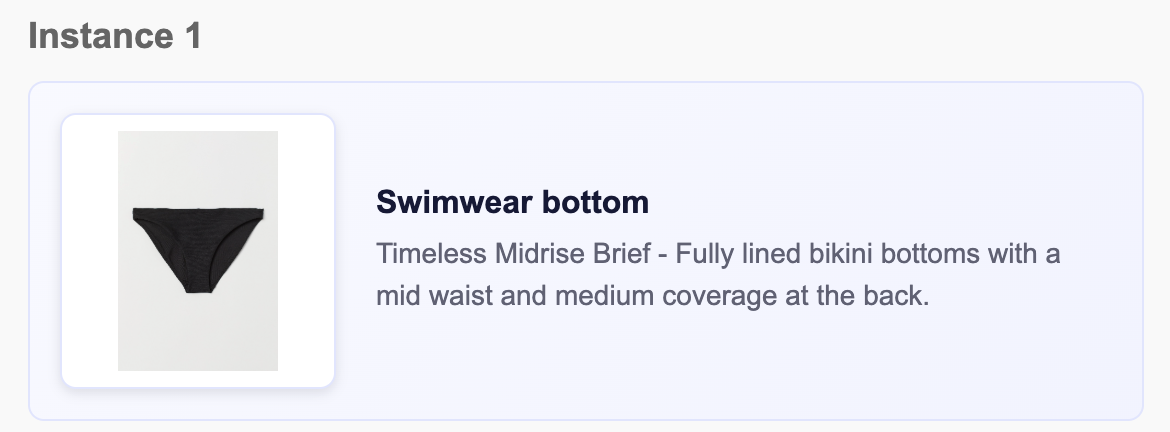}
    \end{subfigure}
    \vspace{-0.8em}
    \begin{subfigure}{0.68\textwidth}
        \includegraphics[width=\linewidth]{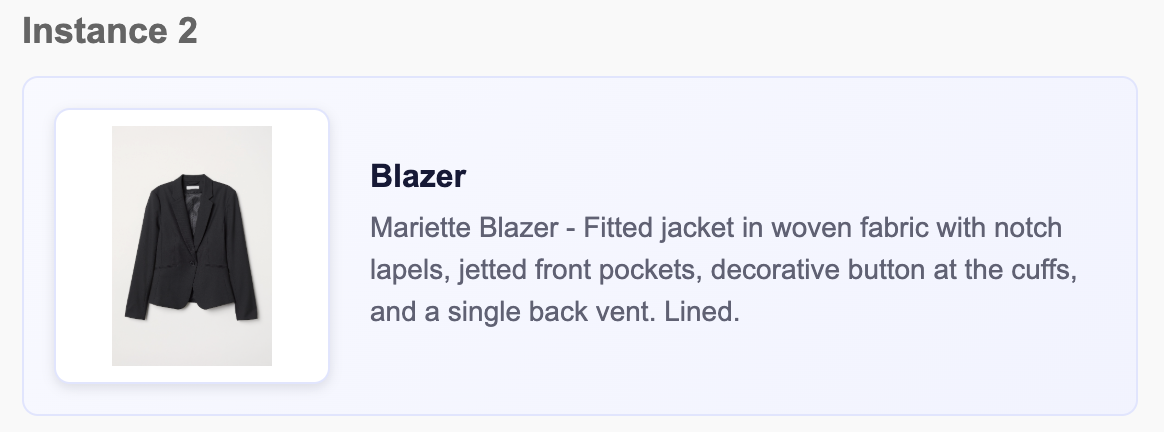}
    \end{subfigure}
    \vspace{-0.8em}
    \begin{subfigure}{0.68\textwidth}
        \includegraphics[width=\linewidth]{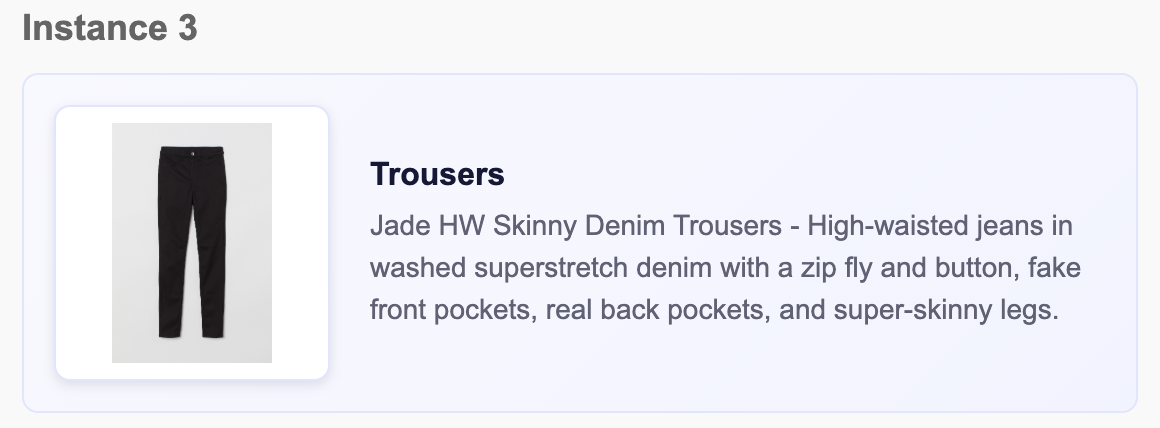}
    \end{subfigure}
    \caption{Product images and descriptions for the three instances used in the human experiments.}
    \label{fig:human_instances_description}
\end{figure}

\begin{figure}
    \centering
    \includegraphics[width=0.78\textwidth]{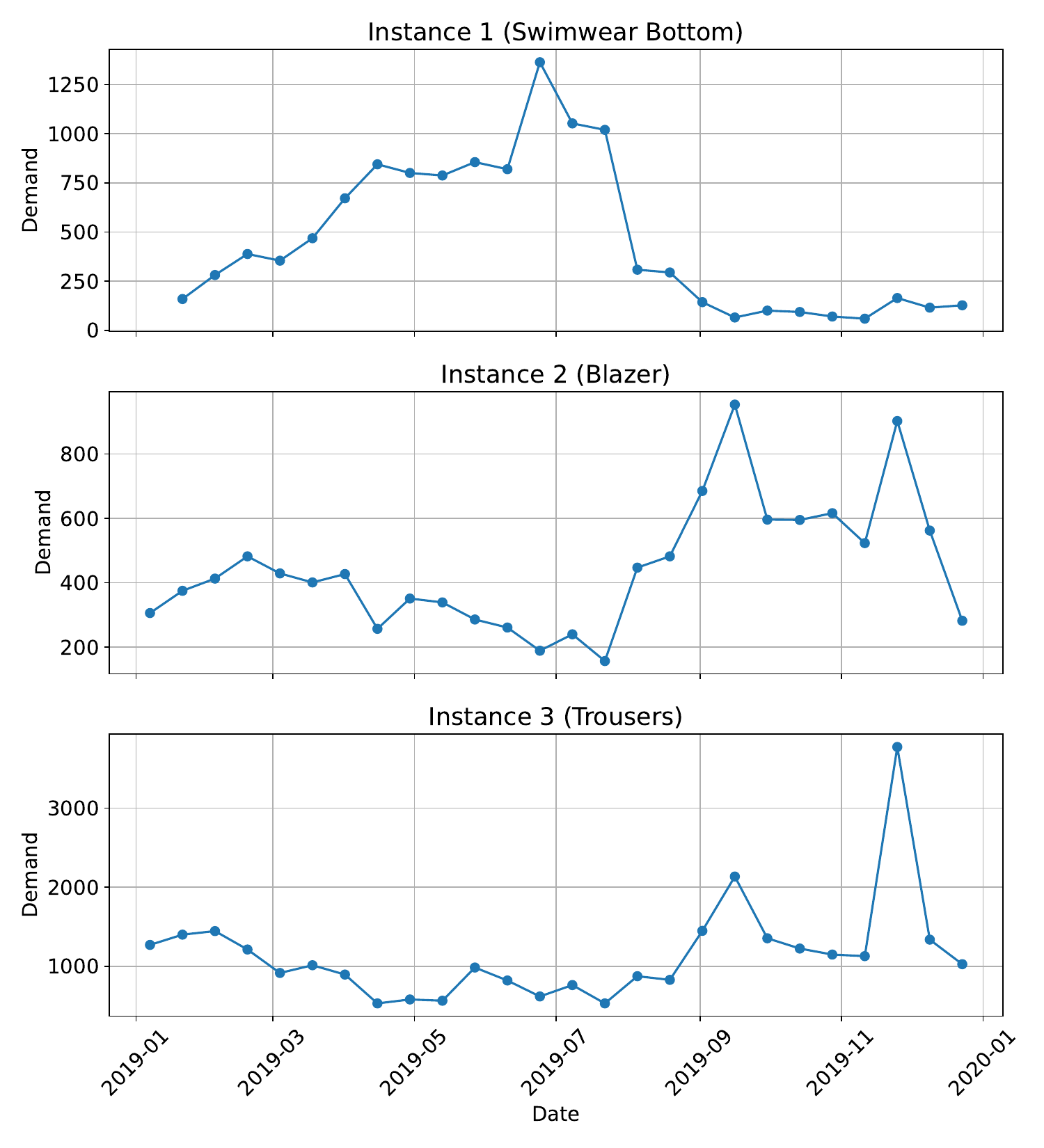}
    \caption{The demand realizations for the three instances used for human experiments. Each panel shows realized demand over time for one instance.
    }
    \label{fig:human_instances_demands}
\end{figure}

\section{Details of Regression Analysis} \label{sec:regression_details}

We assess whether Mode B outperforms an automated baseline by comparing human-in-the-loop Mode B outcomes to automated runs conducted on the same three instances (100 runs per instance). We pool these observations and estimate a linear model with instance fixed effects and a Mode B indicator:
\begin{equation}
\text{Outcome}_{i} = \gamma_0 + \gamma_2\,\mathbbm{1}[\text{instance}=2] + \gamma_3\,\mathbbm{1}[\text{instance}=3] + \delta\,\mathbbm{1}[\text{Mode B}] + u_i,
\end{equation}
where instance 1 is the baseline. The coefficient \(\delta\) captures the average advantage of Mode B over the automated method after controlling for instance differences.

We conduct a one-sided test of \(H_0: \delta \le 0\) versus \(H_1: \delta > 0\). Standard errors are cluster-robust: Mode B observations are clustered by subject, and automated observations are clustered by run. This accounts for within-subject dependence on the human side and repeated structure across automated runs.

\paragraph{Mode A vs.\ OR (Deterministic Algorithm).}
We additionally compare Mode A to a deterministic OR baseline. The analysis mirrors the specification above, replacing the Mode B indicator with a Mode A indicator and using the OR outcomes for each instance. The coefficient on the Mode A indicator captures the average advantage of Mode A over OR after controlling for instance fixed effects. We test \(H_0: \delta \le 0\) vs.\ \(H_1: \delta > 0\) with cluster-robust standard errors, clustering Mode A observations by subject and OR observations by run.

\section{Proof of \texorpdfstring{\cref{thm:ce_lower_bound}}{Theorem 1}} \label{sec:proof_complementarity}

\begin{proof}
\emph{Lower bound.}
Recall that $\mathrm{CE}_i = b(i) - a(i)$. Fix an arbitrary threshold $t \in \bR$. First, observe that if $b(i) > t + \delta$ and $a(i) \leq t$, then $b(i) - a(i) > \delta$, so the event $\{b(i) > t + \delta\} \cap \{a(i) \leq t\}$ is a subset of $\{\mathrm{CE}_i > \delta\}$. Therefore,
\begin{align}
\bP[\mathrm{CE}_i > \delta] \;\geq\; \bP\bigl[b(i) > t + \delta \;\text{and}\; a(i) \leq t\bigr]. \label{eq:subset_bound}
\end{align}
Applying the Fr\'{e}chet inequality ($\bP[A \cap B] \geq \bP[A] + \bP[B] - 1$) with $A = \{b(i) > t + \delta\}$ and $B = \{a(i) \leq t\}$:
\begin{align}
\bP\bigl[b(i) > t + \delta,\; a(i) \leq t\bigr]
&\;\geq\; \bigl(1 - F_b(t + \delta)\bigr) + F_a(t) - 1
\;=\; F_a(t) - F_b(t + \delta). \label{eq:frechet_applied}
\end{align}
Combining \eqref{eq:subset_bound} and \eqref{eq:frechet_applied} and taking the supremum over $t$ yields the bound. The supremum is nonnegative since both CDFs tend to $0$ as $t \to -\infty$.

\emph{Tightness.}
Let $\varepsilon := \sup_{t \in \bR}\{F_a(t) - F_b(t+\delta)\}$.
The case $\varepsilon = 1$ is trivial: then $a \leq t^*$ a.s.\ and $b > t^*+\delta$ a.s.\ for some $t^*$, so $b - a > \delta$ a.s.\ under every coupling.

Assume $\varepsilon < 1$. Define two sub-distributions of total mass $1-\varepsilon$:
\[
F_{a^+}(t) := \max\{F_a(t) - \varepsilon,\, 0\}, \qquad F_{b^-}(t) := \min\{F_b(t),\, 1-\varepsilon\}.
\]
Intuitively, $a^+$ removes the lower $\varepsilon$-mass of $a$ and $b^-$ removes the upper $\varepsilon$-mass of $b$.
By the definition of $\varepsilon$, $F_a(t) - \varepsilon \leq F_b(t+\delta)$ for all $t$, so
\[
F_{a^+}(t) = \max\{F_a(t)-\varepsilon, 0\} \leq F_b(t+\delta),
\]
and since $F_{a^+}(t) \leq 1-\varepsilon$ as well,
\[
F_{a^+}(t) \leq \min\{F_b(t+\delta),\, 1-\varepsilon\} = F_{b^-}(t+\delta).
\]
Normalizing by $1-\varepsilon$, this gives $\widetilde{F}_{a^+}(t) \leq \widetilde{F}_{b^-}(t+\delta)$ for all $t$, meaning the retained distribution of $a$ first-order stochastically dominates the retained distribution of $b - \delta$. By the monotone coupling theorem, the retained masses can therefore be coupled so that $a \geq b - \delta$, equivalently $b - a \leq \delta$, almost surely.

Now couple the remaining mass (total probability $\varepsilon$) arbitrarily. The event $\{b - a > \delta\}$ can occur only within this remaining $\varepsilon$-mass, so $\Pr(b - a > \delta) \leq \varepsilon$. Combined with the lower bound $\Pr(b - a > \delta) \geq \varepsilon$, we obtain $\Pr(b - a > \delta) = \varepsilon$, completing the proof.
\end{proof}

\section{Use of AI Writing Assistance} \label{sec:ai_writing}

Portions of this manuscript were edited and polished with the assistance of large language models (including Claude and GPT-5 Mini). The AI tools were used to help with prose editing, LaTeX formatting, and stylistic improvements. All AI-generated suggestions were reviewed, verified, and revised by the authors, who take full responsibility for the content of this paper.

\end{document}